\documentclass{article}

\PassOptionsToPackage{numbers, compress}{natbib}

\usepackage[main, final]{neurips_2026}

\usepackage[utf8]{inputenc} 
\usepackage[T1]{fontenc}    
\usepackage{hyperref}       
\usepackage{url}            
\usepackage{booktabs}       
\usepackage{amsfonts}       
\usepackage{nicefrac}       
\usepackage{microtype}      
\usepackage{xcolor}         
\usepackage{amsmath}
\usepackage{amsmath,amssymb,amsfonts,bm}
\usepackage{algorithm}
\usepackage{algpseudocode}
\usepackage{graphicx}  
\usepackage[table]{xcolor}
\usepackage{multirow}
\usepackage{caption}
\usepackage{tabularx}
\usepackage{booktabs}
\usepackage{multirow}
\usepackage{makecell}
\usepackage{threeparttable}
\usepackage{siunitx}

\hypersetup{colorlinks=true,linkcolor=black,citecolor=blue!45!black,urlcolor=blue!45!black}

\newcolumntype{L}{>{\raggedright\arraybackslash}X}
\usepackage{amsthm}

\newtheorem{proposition}{Proposition}
\newcommand{\best}[1]
{\mathbf{#1}}
\newcommand{\sbest}[1]{\textcolor{brown}{\mathbf{#1}}}

\newtheorem{lemma}{Lemma}[section]

\title{Flow Annealing Posterior Sampling for Function-Space Regression and Inverse Problems}

\author{%
  Yaozhong Shi\\
  \text{California Institute of Technology} \\
  \texttt{yshi5@caltech.edu} \\
  \And
  Zachary E. Ross\\
  \text{California Institute of Technology}\\
  \texttt{zross@caltech.edu} \\
  \AND
  Yisong Yue\\
  \text{California Institute of Technology}\\
  \texttt{yyue@caltech.edu}\\
}

\begin{document}

\maketitle

\begin{abstract}
Principled regression for stochastic processes is a long-standing challenge with deep connections to scientific inverse problems. We introduce Flow Annealing Posterior Sampling (FLAPS), to our knowledge the first function-space posterior sampling framework that unifies stochastic-process regression and PDE inverse problems. Built on pretrained function-space flow-matching priors, FLAPS enables likelihood-guided posterior inference from sparse and noisy observations, supports variable query discretizations, and avoids explicit prior-density evaluation. Its Langevin correction uses a low-rank covariance preconditioner to exploit dominant function-space correlations across discretizations. Across Gaussian and non-Gaussian stochastic-process regression benchmarks and diverse PDE inverse problems, FLAPS produces coherent posterior samples with well-calibrated uncertainty quantification, significantly outperforming existing functional regression baselines and achieving competitive or better noisy PDE inverse performance than diffusion-based posterior samplers while reducing test-time sampling cost. Our code is available at \url{https://github.com/yzshi5/FLAPS}
\end{abstract}
\section{Introduction}

Stochastic processes are foundational models for distributions over functions and play a central role in functional regression, data assimilation, uncertainty quantification, and scientific inverse problems. In these settings, one observes sparse and noisy measurements of an unknown function or physical field and seeks not only a point prediction, but a posterior distribution over all functions consistent with the observations. Classical Gaussian process (GP) regression provides a principled Bayesian solution when the prior is Gaussian and the observation model is simple~\citep{williams_gaussian_2006}, but many scientific processes are non-Gaussian, high-dimensional, and observed through indirect physical measurements ~\citep{shi_universal_2024, stuart_inverse_2010}.

Scientific inverse problems can be viewed as a natural extension of stochastic-process regression. Instead of directly observing function values, one observes quantities generated by a forward operator, often a PDE solution map. Both problems share the same Bayesian structure: infer an unknown function from partial observations using a prior over functions and a likelihood induced by the measurement process. Recent neural processes learn flexible conditional predictors~\citep{abu_hamad_flow_2026}, neural operators learn efficient PDE solution maps~\citep{li_fourier_2021}, and function-space flow matching learns expressive stochastic-process priors~\citep{shi_stochastic_2026}. However, a general posterior sampling framework that turns pretrained function-space flow-matching priors into likelihood-guided inference for both regression and PDE inverse problems remains missing.

In this work, we introduce \emph{Flow Annealing Posterior Sampling} (FLAPS), a function-space posterior sampling framework built on pretrained function-space flow-matching priors. FLAPS performs posterior inference through a decoupled annealing procedure: samples are transported by the learned flow, corrected by the observation likelihood using Langevin dynamics, and re-bridged across annealing levels. The method avoids explicit prior-density evaluation, supports varying query discretizations, and applies to both direct stochastic-process regression and indirect PDE inverse problems while remaining computationally efficient. To exploit the geometry of function-valued data, FLAPS uses a low-rank covariance preconditioner in the Langevin correction, allowing posterior updates to follow dominant correlations of the underlying function space under sparse observations.

\begin{figure*}[ht]
    \vspace*{-.3cm}
    \centering

    \includegraphics[width=0.95\textwidth]
    {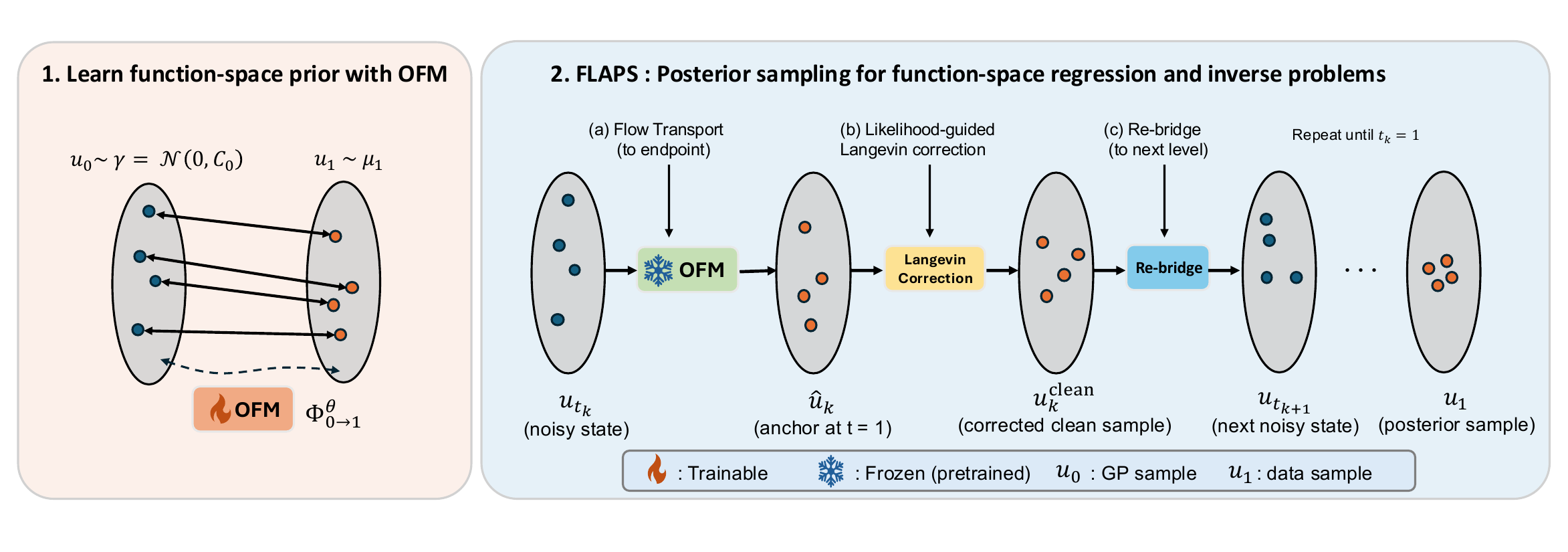}
    \caption{Overview of FLAPS. 
    A pretrained Operator Flow Matching (OFM) prior transports a reference GP to the target stochastic process. 
    Given partial observations, FLAPS freezes the prior and iteratively transports, corrects, and re-bridges samples to obtain posterior samples without explicit prior-density evaluation.}
    \label{fig:faps_overview}  
\end{figure*}

We summarize our main contributions below:
\begin{itemize}
    \item \textbf{First unified function-space flow-matching posterior sampler.}
    To our knowledge, FLAPS is the first posterior sampling framework built on pretrained function-space flow-matching priors that unifies stochastic-process regression and PDE inverse problems. It provides a common likelihood-guided sampler for sparse functional regression and indirect noisy PDE observations.

    \item \textbf{Strong empirical performance.}
    Extensive experiments demonstrate that FLAPS achieves state-of-the-art functional regression performance compared with Neural Process variants, and competitive or better noisy PDE inverse performance than diffusion-based posterior samplers, with lower test-time sampling cost.

    \item \textbf{Low-rank covariance-preconditioned Langevin correction.}
    FLAPS introduces a low-rank covariance preconditioner for Langevin correction, which exploits dominant function-space correlations and significantly improves posterior updates under sparse observations.

    \item \textbf{Flexible and computationally efficient posterior inference.}
    FLAPS is plug-and-play. It leverages pretrained function-space priors and remains backward-compatible with conventional finite-dimensional priors, e.g. rectified flow prior. It supports noisy observations and variable query discretizations without retraining the prior. Empirically, FLAPS is substantially more memory- and time-efficient than likelihood-based posterior sampling with flow priors; see Table~\ref{tab:computational-efficiency}.
\end{itemize}

\section{Related work}

\begin{figure*}[!t]
    \vspace*{-0.3cm}
    \centering
    \includegraphics[width=0.95\textwidth,        trim=0 1.9cm 0 0,
        clip]{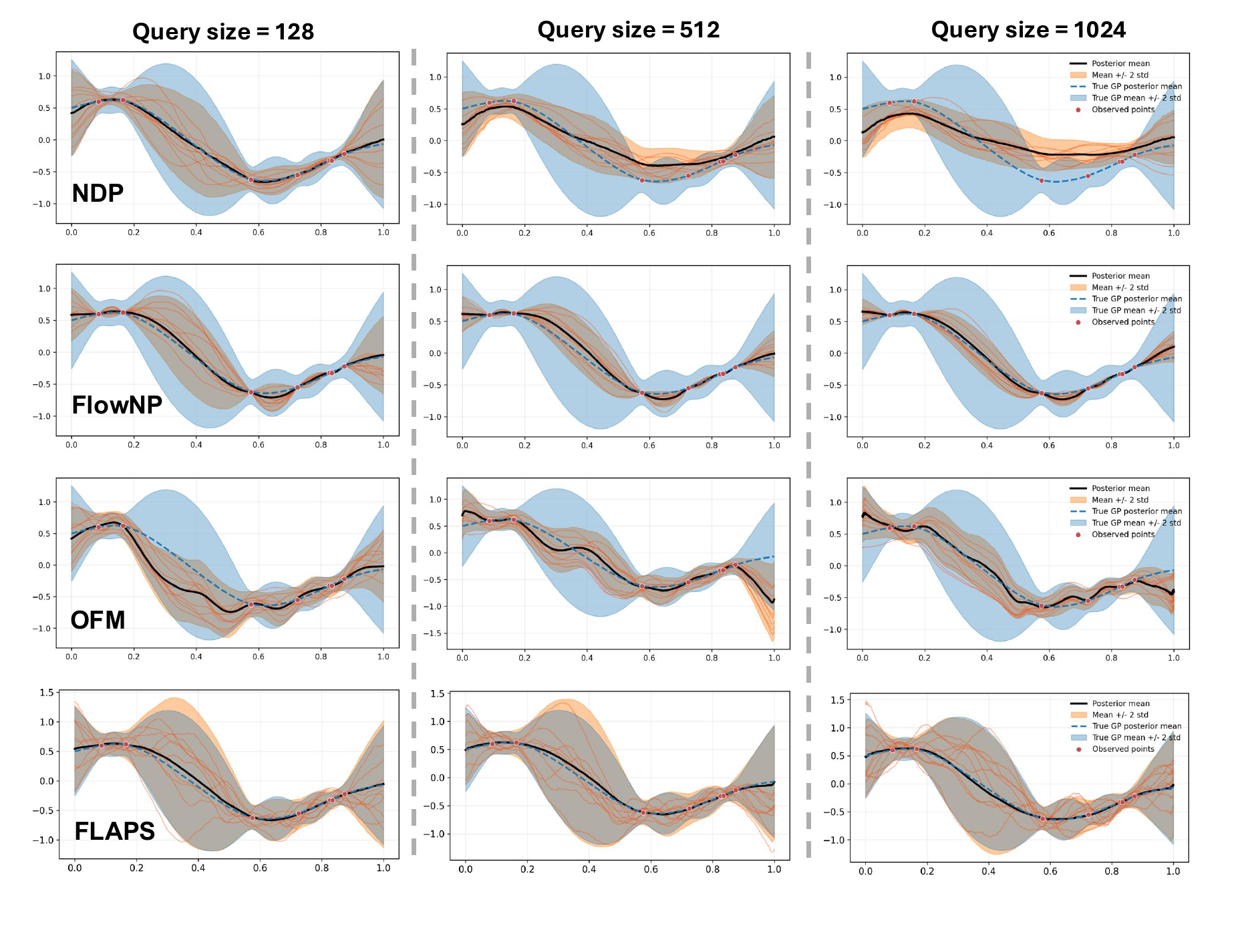}
    \caption{One-dimensional Matérn-kernel GP posterior regression. 
Given seven observations, each method predicts the posterior over query locations in $[0,1]$ at query resolutions 128, 512, and 1024.}
    \label{fig:fig1_GP}
\end{figure*}

\paragraph{Neural operators.}
Neural operators learn mappings between function spaces and have become a core tool for scientific machine learning and PDE surrogate modeling~\citep{li_fourier_2021, kovachki_neural_2023}. Their discretization-flexible formulation makes them well suited for modeling physical fields and solution operators. However, most neural operators are deterministic surrogates and do not directly provide posterior distributions over unknown functions from sparse noisy observations. FLAPS builds on the neural-operator function-space perspective, but targets posterior sampling rather than deterministic operator approximation.

\paragraph{Function-space flow matching.}
Flow matching, stochastic interpolants, and rectified flow learn continuous-time transports between probability distributions and provide efficient alternatives to diffusion models~\citep{lipman_flow_2023, albergo_stochastic_2025, liu_flow_2022, tong_improving_2024}. Recent work extends these ideas from finite-dimensional vectors to functions and stochastic processes. Functional flow matching studies generative modeling directly in function spaces~\citep{kerrigan_functional_2024}. Building on this paradigm,  Operator Flow Matching (OFM) learns stochastic-process priors with neural operators and provides finite-dimensional marginals at arbitrary query sets~\citep{shi_stochastic_2026}. Mesh-informed neural operators further extend this direction beyond regular grids and rectangular domains~\citep{shi_mesh-informed_2025}. These methods primarily focus on learning expressive function-space priors. FLAPS is complementary: it converts a pretrained function-space flow-matching prior into a likelihood-guided posterior sampler for functional regression and PDE inverse problems.

\paragraph{Neural processes.}
Neural Processes learn conditional distributions over functions from context-target pairs~\citep{garnelo_conditional_2018}. Variants such as Attentive Neural Processes, Convolutional Conditional Neural Processes, Neural Diffusion Processes, and Flow Matching Neural Processes improve expressivity, spatial structure, and conditional sample quality~\citep{kim_attentive_2019,  gordon_convolutional_2019, dutordoir_neural_2023, abu_hamad_flow_2026}. However, these methods are primarily amortized conditional models: they learn a direct map from context observations to target distributions. FLAPS instead starts from a pretrained unconditional function-space flow-matching prior and performs likelihood-guided posterior inference at test time, allowing the same prior to be reused across different observation masks, noise levels, query discretizations, and inverse-problem likelihoods. These differences are summarized in Table~\ref{tab:faps_comparison_compact}.

\paragraph{Generative models for PDE solving.}
Generative models have recently been used for PDE solving, uncertainty quantification, and inverse problems. In particular, diffusion-based posterior samplers combine learned priors with observation guidance: DAPS improves inverse problem solving through decoupled noise annealing~\citep{zhang_improving_2025}, FunDPS develops guided diffusion sampling on function spaces~\citep{yao_guided_2026}, and DDIS decouples learned coefficient priors from neural-operator forward models for inverse PDE problems~\citep{lin_decoupled_2026}. BLADE performs derivative-free Bayesian inversion with diffusion priors~\citep{zheng_blade_2026, zheng_inversebench_2025}. These methods demonstrate the value of generative priors for scientific inverse problems, but are primarily diffusion-based and focused on PDE settings. FLAPS provides a flow-matching-prior alternative that unifies stochastic-process regression and PDE inverse problems through flow transport, likelihood-guided Langevin correction, and re-bridging in function space.

\begin{figure*}[bt]
    \vspace*{-0.3cm}
    \centering
    \includegraphics[width=1.0\textwidth,trim=0 5.4cm 0 0, clip]{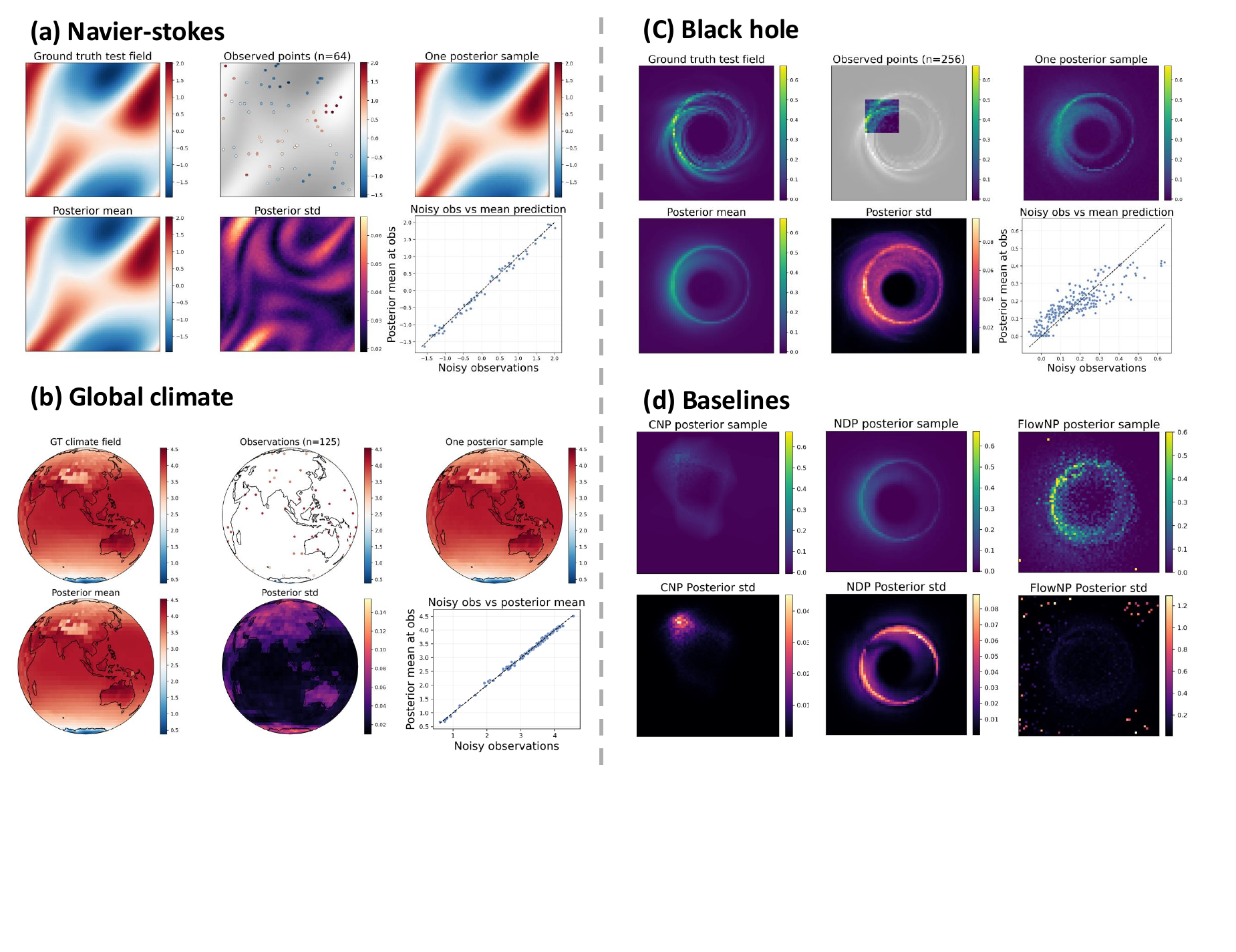}
    \caption{ Non-Gaussian functional regression . Given partial and noisy observation, 
    FLAPS is evaluated on (a) Navier--Stokes flow fields, (b) global climate fields on the sphere, and (c) black-hole imaging data. 
     (d) compares representative NP baselines posterior samples}
    \label{fig:non_gaussian_regression}

\end{figure*}

\section{Method}
\label{sec:method}

We present FLAPS in the function-space setting. Let \(u_1\in\mathcal U\) denote an
unknown function or physical field on a domain \(D\subset\mathbb R^{d_x}\),
with pointwise values \(u_1(x)\in\mathbb R^{d_u}, \text{where} ~x \in D\). Given sparse and noisy
observations \(y\), our goal is to sample from the posterior distribution over
\(u_1\).

\subsection{Function-space flow matching prior}
\label{sec:function_space_prior}

FLAPS builds on a pretrained operator flow-matching prior~\citep{shi_stochastic_2026}. Let
\(\gamma=\mathcal N(0,\mathcal C_0)\) be a Gaussian measure on \(\mathcal U\),
where \(\mathcal C_0\) is a self-adjoint, positive, trace-class covariance
operator. Let \(\mu_1\) denote the target probability measure of the data
functions. The goal of prior learning is to construct a probability flow that
transports \(\gamma\) to \(\mu_1\). In this work, the pretrained prior is learned using the independent-coupling
flow-matching objective~\citep{tong_improving_2024}. We draw independent endpoint samples $
    u_0\sim\gamma,~
    u_1\sim\mu_1,$
and define the noisy straight interpolation in function space
\begin{equation}
    u_t
    =
    t u_1
    +
    (1-t)u_0
    +
    \sigma_{\min}\xi,
    \label{eq:function_independent_path}
\end{equation}

Here, $\sigma_{\min}>0$ is a small constant that prevents the training path from becoming singular near $t=1$, and $\xi\sim\gamma$. Setting $\sigma_{\min}=0$ recovers the rectified flow training objective~\citep{liu_flow_2022}. Since the deterministic part of
Eq.~\eqref{eq:function_independent_path} is linear in \(t\), the
conditional velocity given the sampled pair \((u_0,u_1)\) is
\begin{equation}
    v_t(u_t\mid u_0,u_1)=u_1-u_0 .
    \label{eq:function_target_velocity}
\end{equation}
The corresponding marginal velocity field is obtained by averaging over the
conditional path:
\begin{equation}
    v_t^\dagger(u)
    =
    \mathbb E\left[u_1-u_0\mid u_t=u\right].
    \label{eq:function_marginal_velocity}
\end{equation}
We parameterize the marginal velocity field by a neural operator
\(v_\theta:[0,1]\times\mathcal U\to\mathcal U\) and train it with
\begin{equation}
    \mathcal L_{\rm FM}(\theta)
    =
    \mathbb E_{t\sim\mathrm{Unif}(0,1),u_0,u_1,\xi}
    \left[
    \left\|
        v_\theta(t,u_t)
        -
        (u_1-u_0)
    \right\|_{\mathcal U}^2
    \right].
    \label{eq:function_fm_objective}
\end{equation}


After training, \(v_\theta\) defines a probability-flow ~\citep{shi_stochastic_2026, kerrigan_functional_2024} and for \(0\le s\le t\le 1\), we denote by
\(\Phi^\theta_{s\to t}:\mathcal U\to\mathcal U\) the solution map of this ODE :
given an initial state \(u_s\in\mathcal U\) at time \(s\),
\[
    \Phi^\theta_{s\to t}(u_s)= u_s + \int_s^t v_\theta(\tau, u_\tau) d\tau = u_t .
\]

For posterior computation, we work with finite-dimensional marginals induced by
point evaluations. For a query set \(X=\{x_i\}_{i=1}^n\), define
\[
    \Pi_X u_t
    =
    u_{t, X}
    =
    \big(u_t(x_1),\ldots,u_t(x_n)\big)   \in
    (\mathbb R^{d_u})^n
    \cong
    \mathbb R^{n d_u}.
\]
The Gaussian reference measure induces
\begin{equation}
    \gamma_X
    =
    (\Pi_X)_\#\gamma
    =
    \mathcal N(0,\Sigma_0^X),
    \label{eq:gamma_marginal}
\end{equation}
where \(\Sigma_0^X\) is obtained by evaluating the covariance operator
\(\mathcal C_0\) on \(X\). Similarly, the target process induces
\(\mu_{1,X}=(\Pi_X)_\#\mu_1\). As shown in the stochastic-process construction
of OFM~\citep{shi_stochastic_2026}, these finite-dimensional marginals are
consistent across query sets and define a process-level prior.
Thus, the learned OFM prior can be evaluated on arbitrary query sets \(X\). The independent-coupling path also induces the bridge (transition) kernel used in FLAPS.
Conditioned on a clean endpoint \(u_{1,X}\), marginalizing over
\(u_{0,X}\sim\gamma_X\) and \(\xi_X\sim\gamma_X\) in
Eq.~\eqref{eq:function_independent_path} gives
\begin{equation}
    q_t(u_{t,X}\mid u_{1,X})
    =
    \mathcal N
    \left(
        t u_{1,X},
        s(t)^2\Sigma_0^X
    \right),
    \qquad
    s(t)=\sqrt{(1-t)^2+\sigma_{\min}^2} \approx 1-t.
    \label{eq:independent_bridge_kernel}
\end{equation}
This bridge kernel is used later to re-bridge corrected clean samples across annealing levels. 

\subsection{Unified posterior formulation}
\label{sec:unified_posterior}

We write both functional regression and PDE inverse problems using the
observation model
\begin{equation}
    y=\mathcal A(u_1)+\epsilon,
    \qquad
    \epsilon\sim\mathcal N(0,\sigma_y^2 I),
    \label{eq:obs_model}
\end{equation}
where \(\mathcal A:\mathcal U\to\mathbb R^m\) is a task-dependent observation
operator and, once again, \(u_1\) denotes a clean function drawn from the target process.
The posterior endpoint law is
\begin{equation}
    \pi_1^y(du_1)
    \propto
    p(y\mid u_1)\,\mu_1(du_1),
    \qquad
    p(y\mid u_1)
    \propto
    \exp\left(
    -\frac{1}{2\sigma_y^2}
    \|y-\mathcal A(u_1)\|_2^2
    \right).
    \label{eq:posterior_measure}
\end{equation}
On a finite query set \(X=\{x_i\}_{i=1}^n\), this corresponds to posterior
inference over $
    u_{1,X}
    =
    \big(u_1(x_1),\ldots,u_1(x_n)\big)
$. For functional regression, the observations are direct noisy evaluations of the
unknown function. Let \(P_\Omega\) denote the masking or point-evaluation
operator on observed locations \(\Omega\subset X\). Then
\begin{equation}
    \mathcal A(u_1)=P_\Omega u_1,
    \qquad
    y=P_\Omega u_1+\epsilon.
    \label{eq:regression_operator}
\end{equation}
For PDE inverse problems, \(u_1\) is an unknown input field, such as a
coefficient, source, or initial condition, and the observations are sparse
measurements of a PDE response. Let \(\mathcal G_\phi\) be a differentiable PDE
solver or a pretrained neural-operator surrogate for the forward PDE solution
map. Then
\begin{equation}
    \mathcal A(u_1)=P_\Omega \mathcal G_\phi(u_1),
    \qquad
    y=P_\Omega\mathcal G_\phi(u_1)+\epsilon.
    \label{eq:pde_operator}
\end{equation}
Thus, regression and PDE inverse problems differ only through the observation
operator \(\mathcal A\).

\subsection{Flow Annealing Posterior Sampling}
\label{sec:faps}

FLAPS samples from Eq.~\eqref{eq:posterior_measure} by alternating between
flow-based prior transport, likelihood-guided correction, and re-bridging.
Let \(0=t_0<t_1<\cdots<t_K=1\) be an annealing schedule. For a finite query set
\(X\), let \(q_t^X(u_t\mid u_1)\) denote the bridge kernel (shown in Eq.~\eqref{eq:independent_bridge_kernel}) induced by the
independent-coupling flow path.
We define the measurement-conditioned annealed marginal
\begin{equation}
    \pi_t^{y,X}(du_t)
    =
    \int q_t^X(du_t\mid u_1)\,\pi_1^{y,X}(du_1),
    \label{eq:annealed_posterior}
\end{equation}
where \(\pi_1^{y,X}\) is the endpoint posterior over clean functions on \(X\).

\begin{proposition}[Annealing and re-bridging]
\label{prop:faps_rebridge}
Assume \(u_{t_k}\sim \pi_{t_k}^{y,X}\). If we sample
\[
    u_1\sim \pi_1^{y,X}(du_1\mid u_{t_k}),
\]
and then re-bridge
\[
    u_{t_{k+1}}\sim q_{t_{k+1}}^X(du_{t_{k+1}}\mid u_1),
\]
then \(u_{t_{k+1}}\sim \pi_{t_{k+1}}^{y,X}\).
\end{proposition}

Proposition~\ref{prop:faps_rebridge} shows that exact sampling from the clean
conditional \(\pi_1^{y,X}(du_1\mid u_{t_k})\), followed by re-bridging, preserves
the desired annealed posterior marginal. The proof is provided in the Appendix~\ref{app:faps_proof_algorithm}.
In practice, this clean conditional is intractable because the learned
function-space prior is implicit. FLAPS therefore approximates it using an OFM
endpoint anchor and likelihood-guided Langevin correction. At time \(t_k\), given a bridge state \(u_{t_k}\), we integrate the pretrained
flow to the endpoint:
\begin{equation}
    \hat u_k
    =
    \Phi^\theta_{t_k\to1}(u_{t_k}).
    \label{eq:anchor}
\end{equation}
The endpoint \(\hat u_k\) acts as a clean prior-consistent anchor associated
with the current bridge state. In the following, we suppress the subscript \(X\) when no ambiguity arises; all variables in the practical sampler are understood as finite-dimensional evaluations on the query set \(X\). Around this anchor, inspired by~\citep{zhang_improving_2025}, we use a local
Gaussian approximation
\begin{equation}
    p(u_1\mid u_{t_k})
    \approx
    \mathcal N(\hat u_k,\lambda_k^2 C_X),
    \qquad
    \lambda_k=\max\big(\lambda_{\min},
        \lambda_{\mathrm{scale}}(1-t_k)\big).
    \label{eq:local_anchor}
\end{equation}
where \(C_X\) is an empirical covariance preconditioner on the query set, $\lambda_{\min}$ is a small constant and $\lambda_{\rm scale}$ equals 1 by default. The
local posterior correction target is
\begin{equation}
    \widetilde\pi_k(u_1)
    \propto
    p(y\mid u_1)\,
    \mathcal N(u_1;\hat u_k,\lambda_k^2 C_X).
    \label{eq:local_posterior}
\end{equation}

Starting from \(u_1^{(0)}=\hat u_k\), FLAPS performs \(L\) Langevin correction
steps:
\begin{equation}
    u_1^{(\ell+1)}
    =
    u_1^{(\ell)}
    +
    \eta
    \left[
        -\frac{u_1^{(\ell)}-\hat u_k}{\lambda_k^2}
        +
         C_X\nabla_{u_1}\log p(y\mid u_1^{(\ell)})
    \right]
    +
    \sqrt{2\eta}\,\zeta^{(\ell)},
    \qquad
    \zeta^{(\ell)}\sim\mathcal N(0,C_X).
    \label{eq:faps_langevin}
\end{equation}
Here, \(\eta\) denotes the Langevin step size. The covariance \(C_X\) preconditions the likelihood gradient
and defines the Langevin noise covariance. The anchor term is kept explicit as
\(- (u_1-\hat u_k)/\lambda_k^2\), which pulls the sample toward the flow-predicted clean endpoint. The likelihood gradient is
\begin{equation}
    \nabla_{u_1}\log p(y\mid u_1)
    =
    \frac{1}{\sigma_y^2}
    J_{\mathcal A}(u_1)^\ast
    \big(y-\mathcal A(u_1)\big),
    \label{eq:likelihood_gradient}
\end{equation}
where \(J_{\mathcal A}(u_1)^\ast\) is the adjoint of the Fréchet derivative of
\(\mathcal A\). For direct regression, this becomes
\begin{equation}
    \nabla_{u_1}\log p(y\mid u_1)
    =
    \frac{1}{\sigma_y^2}
    P_\Omega^\top
    (y-P_\Omega u_1).
    \label{eq:regression_gradient}
\end{equation}
For PDE inverse problems with \(\mathcal A=P_\Omega\mathcal G_\phi\), it becomes
\begin{equation}
    \nabla_{u_1}\log p(y\mid u_1)
    =
    \frac{1}{\sigma_y^2}
    J_{\mathcal G_\phi}(u_1)^\ast
    P_\Omega^\top
    \big(y-P_\Omega\mathcal G_\phi(u_1)\big),
    \label{eq:pde_gradient}
\end{equation}
which is computed by automatic differentiation through the pretrained neural
operator \(\mathcal G_\phi\). After \(L\) Langevin steps, we obtain \(u_k^{\rm clean}=u_1^{(L)}\). We then
re-bridge the corrected clean sample to the next annealing level:
\begin{equation}
    u_{t_{k+1}}
    =
    \begin{cases}
    t_{k+1}u_k^{\rm clean}
    +
    s(t_{k+1})\xi_{k+1},
    & t_{k+1}<1,\\
    u_k^{\rm clean},
    & t_{k+1}=1,
    \end{cases}
    \qquad
    \xi_{k+1}\sim\mathcal N(0,\Sigma_0^X),
    \label{eq:rebridge}
\end{equation}
Repeating this transport--correction--rebridging procedure from \(t_0=0\) to
\(t_K=1\) yields approximate posterior samples from \(p(u_1\mid y)\). The learned
prior density is never explicitly evaluated.
\subsection{Empirical low-rank covariance preconditioning}
\label{sec:low_rank_cov}

The covariance preconditioner \(C_X\) in Eq.~\eqref{eq:faps_langevin} is
estimated from clean samples of the learned target process, not from the initial
Gaussian reference samples. Crucially, this estimation is performed entirely offline as a one-time preprocessing step using the unconditional prior. Once calculated, $C_X$ is frozen and reused across all subsequent test-time observation masks and noise realizations, requiring zero online re-estimation overhead. Specifically, we draw
\(u_0^{(j)}\sim\gamma_X\), transport them through the pretrained flow, and obtain
\begin{equation}
    u_1^{(j)}
    =
    \Phi^\theta_{0\to1}(u_0^{(j)}),
    \qquad
    j=1,\ldots,N_c.
    \label{eq:clean_cov_samples}
\end{equation}
We then compute the empirical covariance
\begin{equation}
    \widehat C_X
    =
    \frac{1}{N_c-1}
    \sum_{j=1}^{N_c}
    \left(u_1^{(j)}-\bar u_1\right)
    \left(u_1^{(j)}-\bar u_1\right)^\top,
    \qquad
    \bar u_1=\frac{1}{N_c}\sum_{j=1}^{N_c}u_1^{(j)}.
    \label{eq:empirical_covariance}
\end{equation}
For high-dimensional fields, we use a low-rank approximation
\begin{equation}
    C_X
    =
    Q_r\Lambda_rQ_r^\top+\sigma_{\rm res}^2 I,
    \label{eq:low_rank_covariance}
\end{equation}
where \(Q_r\) and \(\Lambda_r\) contain the leading empirical eigenvectors and
eigenvalues, and the residual diagonal term stabilizes directions outside the
dominant subspace. In implementation, a Cholesky factor of \(C_X\) is used to
apply covariance matrix-vector products and to sample the Langevin noise
\(\zeta\sim\mathcal N(0,C_X)\). This preconditioner captures dominant correlations of the learned target
process. Under sparse observations, the raw likelihood gradient is localized to
observed coordinates or to the adjoint of a sparse PDE observation operator.
Multiplication by \(C_X\) propagates this information along correlated
function-space modes, significantly improving posterior correction in unobserved regions; see Appendix~\ref{app:lowrank-cov-ablation} for a detailed ablation and scaling study.

\section{Experiments}
\label{sec:experiments}

\begin{figure*}[ht]
    \centering
\includegraphics[width=1.0\textwidth,trim=0 3.8cm 0 0, clip]{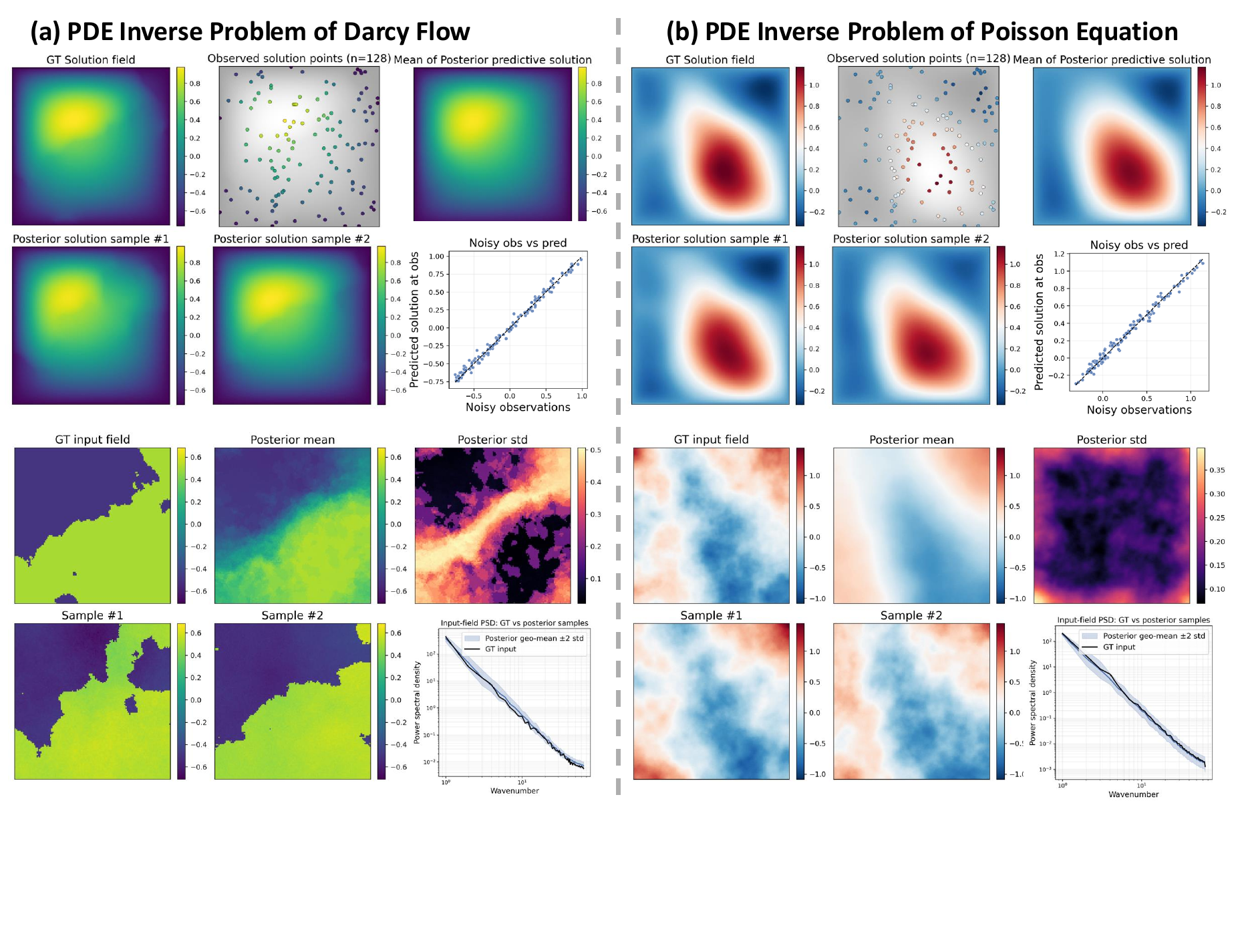}
\caption{PDE inverse problems with 128 noisy solution observations ($0.78\%$) on a $128\times128$ resolution. For (a) Darcy flow and (b) the Poisson equation, FLAPS infers posterior input fields and corresponding posterior predictive solution fields from sparse noisy measurements.}

    \label{fig:pde_inverse}
\end{figure*}

We evaluate FLAPS as a posterior sampler for pretrained function-space flow-matching priors under sparse and noisy observations. The experiments cover three settings: one-dimensional GP regression with exact reference posteriors, non-Gaussian functional regression on scientific field data, and PDE inverse problems with indirect solution observations. Across all experiments, FLAPS uses a pretrained prior and performs posterior inference at test time without retraining for new observation masks or noise realizations. Full dataset details, observation settings, and evaluation metrics are provided in Appendix~\ref{sec:experimental_setup}.

\subsection{GP Functional Regression}
\label{sec:gp_regression}

We first consider one-dimensional GP regression, where the exact posterior is available and therefore provides a controlled benchmark for posterior sampling. We evaluate both a stationary Matérn GP and a nonstationary Gibbs-kernel GP. For each test function, we observe a small set of noisy context values and sample the posterior distribution over the full query grid at resolutions \(128\) and \(512\). Figure~\ref{fig:fig1_GP} shows representative Matérn posterior samples. FLAPS closely matches the reference posterior mean and uncertainty across query resolutions. In contrast, neural-process baselines tend to over-smooth the posterior or produce miscalibrated uncertainty and their performance degrades at higher query resolutions, reflecting that marginal consistency is only approximated rather than rigorously enforced. Direct OFM posterior sampling is also less accurate under sparse observations.

Tables~\ref{tab:FAPS_GP_matern} and~\ref{tab:FAPS_GP_Gibbs} report posterior-distribution metrics over the test set, comparing FLAPS with Transformer Neural Processes (TNP;~\citep{nguyen_transformer_2022}), Conditional Neural Processes (CNP;~\citep{garnelo_conditional_2018}), Attentive Neural Processes (ANP;~\citep{kim_attentive_2019}), Neural Diffusion Processes (NDP;~\citep{dutordoir_neural_2023}), Flow Matching Neural Processes (FlowNP;~\citep{abu_hamad_flow_2026}), and OFM posterior sampling~\citep{shi_stochastic_2026}. We use Sliced Wasserstein Distance (SWD) and Maximum Mean Discrepancy (MMD), which directly compare generated posterior samples with reference posterior samples; lower values indicate better posterior agreement. FLAPS achieves the best performance across both GP priors and query resolutions, demonstrating that likelihood-guided flow annealing improves the full posterior distribution, not merely pointwise reconstruction.

\subsection{Non-Gaussian Functional Regression}
\label{sec:high_dim_regression}

We next evaluate direct functional regression on complex non-Gaussian fields, including Navier--Stokes vorticity, black-hole imaging data, and global climate fields. Unlike the GP benchmarks, exact posterior distributions are unavailable for these datasets. We therefore assess both probabilistic calibration and reconstruction quality. We use CRPS (Continuous Ranked Probability Score) to measure the quality of the predictive posterior distribution, SSR (Spread-Skill Ratio) to assess calibration between posterior spread and prediction error, PSNR (Peak Signal-to-Noise Ratio) and SSIM (Structural Similarity Index) to measure reconstruction fidelity, and Relative \(L^2\) to measure normalized reconstruction error. For CRPS and Relative \(L^2\), lower is better; for PSNR and SSIM, higher is better; for SSR, values closer to one indicate better calibration.

Figure~\ref{fig:non_gaussian_regression} shows representative posterior means, posterior samples, and uncertainty maps. Across all three datasets, FLAPS reconstructs the dominant spatial structures from sparse observations while producing coherent posterior samples. For Navier--Stokes, FLAPS captures large-scale flow patterns; for black-hole imaging, it preserves the ring-like morphology; and for global climate, it produces smooth posterior reconstructions on the spherical domain. The observation-versus-prediction plots show that the posterior mean is consistent with the noisy observations. In contrast, the baselines often over-smooth the posterior mean, underestimate uncertainty, or generate samples with degraded spatial structure.

Table~\ref{tab:reg_result_nonGP} reports quantitative comparisons. FLAPS achieves the strongest overall performance, with the best or near-best calibration and reconstruction metrics across Navier--Stokes, black-hole, and global climate datasets. These results show that FLAPS scales beyond analytically tractable GP posteriors to complex learned stochastic-process priors.

\begin{table*}[h]
\centering
\caption{One-dimensional Matérn GP regression. Lower SWD and MMD indicate
better agreement with the reference posterior. Values are mean
\(\pm\) standard deviation. Best values are in bold.}
\label{tab:FAPS_GP_matern}

\scriptsize
\setlength{\tabcolsep}{2.5pt}
\renewcommand{\arraystretch}{1.08}

\resizebox{\textwidth}{!}{%
\begin{tabular}{@{}lcccc@{}}
\toprule
& \multicolumn{2}{c}{Query size \(128\)}
& \multicolumn{2}{c}{Query size \(512\)} \\
\cmidrule(lr){2-3}
\cmidrule(lr){4-5}
Algorithm $\downarrow$ Metric $\rightarrow$& SWD & MMD & SWD & MMD \\
\midrule

TNP
& \(7.04{\cdot}10^{-1}\pm3.7{\cdot}10^{-1}\)
& \(5.79{\cdot}10^{-1}\pm1.8{\cdot}10^{-1}\)
& \(7.24{\cdot}10^{-1}\pm3.9{\cdot}10^{-1}\)
& \(5.52{\cdot}10^{-1}\pm2.0{\cdot}10^{-1}\) \\

CNP
& \(2.02{\cdot}10^{-1}\pm8.1{\cdot}10^{-2}\)
& \(1.87{\cdot}10^{-1}\pm7.7{\cdot}10^{-2}\)
& \(2.01{\cdot}10^{-1}\pm7.9{\cdot}10^{-2}\)
& \(1.87{\cdot}10^{-1}\pm7.8{\cdot}10^{-2}\) \\

ANP
& \(2.04{\cdot}10^{-1}\pm5.3{\cdot}10^{-2}\)
& \(1.56{\cdot}10^{-1}\pm6.6{\cdot}10^{-2}\)
& \(2.03{\cdot}10^{-1}\pm5.2{\cdot}10^{-2}\)
& \(1.61{\cdot}10^{-1}\pm6.7{\cdot}10^{-2}\) \\

NDP
& \(2.65{\cdot}10^{-1}\pm1.6{\cdot}10^{-1}\)
& \(2.78{\cdot}10^{-1}\pm1.8{\cdot}10^{-1}\)
& \(4.65{\cdot}10^{-1}\pm2.1{\cdot}10^{-1}\)
& \(5.08{\cdot}10^{-1}\pm2.2{\cdot}10^{-1}\) \\

FlowNP
& \(2.24{\cdot}10^{-1}\pm9.4{\cdot}10^{-2}\)
& \(2.14{\cdot}10^{-1}\pm1.1{\cdot}10^{-1}\)
& \(2.66{\cdot}10^{-1}\pm7.8{\cdot}10^{-2}\)
& \(2.19{\cdot}10^{-1}\pm8.6{\cdot}10^{-2}\) \\

OFM
& \(2.17{\cdot}10^{-1}\pm1.1{\cdot}10^{-1}\)
& \(2.16{\cdot}10^{-1}\pm1.0{\cdot}10^{-1}\)
& \(2.92{\cdot}10^{-1}\pm1.2{\cdot}10^{-1}\)
& \(2.87{\cdot}10^{-1}\pm1.3{\cdot}10^{-1}\) \\

\rowcolor{yellow!25}
\textbf{FLAPS (Ours)}
& \(\mathbf{1.42{\cdot}10^{-1}}\pm9.0{\cdot}10^{-2}\)
& \(\mathbf{1.28{\cdot}10^{-1}}\pm9.5{\cdot}10^{-2}\)
& \(\mathbf{1.47{\cdot}10^{-1}}\pm1.0{\cdot}10^{-1}\)
& \(\mathbf{1.36{\cdot}10^{-1}}\pm1.0{\cdot}10^{-1}\) \\

\bottomrule
\end{tabular}%
}
\end{table*}

\begin{table}[h]

\caption{ Non-Gaussian functional regression. Lower CRPS is better, SSR is best near 1. For PSNR and SSIM, higher is better. Relative $L^2$ is lower better. Best values are bolded.}

\footnotesize
\centering
\renewcommand{\arraystretch}{0.9}
\setlength{\tabcolsep}{3.pt} %
{
\begin{tabular}
{@{}c  !{\vrule width 0.5pt}  c !{\vrule width 0.5pt}cccccc@{}}
\toprule
Datasets & Method & CRPS $\downarrow$  & SSR $\rightarrow 1$  & PSNR $\uparrow$ & SSIM $\uparrow$ & Relative $L^2 \downarrow$ \\ 
\midrule
\multirow{7}{*}{Navier-Stokes}
& CNP &  $ 7.37 \cdot 10^{-2}  $& $7.53 \cdot 10^{-1} $ & $2.60 \cdot 10^{~1}$&  $6.96 \cdot 10^{-1}$ & $2.10 \cdot 10^{-1}$  \\
& ANP & $ 8.07 \cdot 10^{-2}$& $8.90 \cdot 10^{-1} $ & $2.83 \cdot 10^{~1}$ &  $8.48 \cdot 10^{-1}$ & $1.55 \cdot 10^{-1}$ \\
& NDP & $ 3.81 \cdot 10^{-1}$& $1.59 \cdot 10^{-1} $ & $1.61 \cdot 10^{~1}$ &  $4.81 \cdot 10^{-1}$ & $5.87 \cdot 10^{-1}$ \\
& FlowNP & $ 1.26 \cdot 10^{-1}$& $ 1.28 \cdot 10^{~0} $ & $1.99 \cdot 10^{~1}$ &  $4.47 \cdot 10^{-1}$ & $3.99 \cdot 10^{-1}$ \\ 
& OFM & $ 2.97 \cdot 10^{-2}$& $ 1.04 \cdot 10^{~0} $ & $\best{3.52 \cdot 10^{1}}$ &  $\best{9.01 \cdot 10^{-1}}$ & $1.09 \cdot 10^{-1}$ \\ 
&   FLAPS (Ours) &    $ \best{2.79 \cdot 10^{-2}}$& $ \best{1.01 \cdot 10^{~0}}$ & $ 3. 43 \cdot 10^{1}$ &  $9.09 \cdot 10^{-1}$ & $\best{1.06 \cdot 10^{-1}}$
\\ \midrule
\multirow{7}{*}{Black hole}
& CNP & $ 2.18 \cdot 10^{-2} $& $1.31 \cdot 10^{~0} $ & $1.18 \cdot 10^{~1}$ &  $1.39 \cdot 10^{-1}$ & $1.58 \cdot 10^{~0}$  \\
& ANP & $ 1.93 \cdot 10^{-2}$& $1.51 \cdot 10^{-1} $ & $1.97 \cdot 10^{~1}$ &  $3.48 \cdot 10^{-1}$ & $6.40 \cdot 10^{-1}$ \\
& NDP & $ 2.30 \cdot 10^{-2}$& $1.81 \cdot 10^{-1} $ & $1.68 \cdot 10^{~1}$ &  $4.03\cdot 10^{-1}$ & $8.59 \cdot 10^{-1}$ \\
& FlowNP & $ 1.49 \cdot 10^{-2}$& $\best{7.88 \cdot 10^{-1}}$ & $1.83 \cdot 10^{~1}$ &  $2.11 \cdot 10^{-1}$ & $7.58 \cdot 10^{-1}$ \\
& OFM & $ 1.33 \cdot 10^{-2}$& $3.62 \cdot 10^{-1}$ & $1.82 \cdot 10^{~1}$ &  $4.04 \cdot 10^{-1}$ & $7.40 \cdot 10^{-1}$  \\
& FLAPS (Ours)& $ \best{1.26 \cdot 10^{-2}}$& $7.27 \cdot 10^{-1}$ & $\best{2.05 \cdot 10^{~1}}$ &  $\best{5.48 \cdot 10^{-1}}$ & $\best{5.71 \cdot 10^{-1}}$\\ \midrule
\multirow{6}{*}{Global Climate}
&CNP & $ 8.91 \cdot 10^{-2}  $& $\best{9.70 \cdot 10^{-1}} $ & $2.37 \cdot 10^{~1}$&  $6.48 \cdot 10^{-1}$ & $8.23 \cdot 10^{-2}$  \\
&ANP & $ 5.98\cdot 10^{-2}$& $7.10 \cdot 10^{-2} $ &  $3.18 \cdot 10^{1}$ & $9.14 \cdot 10^{-1}$ & $3.22 \cdot 10^{-2}$ \\
&NDP & $1.45 \cdot 10^{-1} $& $2.87 \cdot 10^{-1} $ & $2.51 \cdot 10^{~1}$ &$8.99 \cdot 10^{-1}$ & $7.77 \cdot 10^{-2}$ \\
&FlowNP & $ 1.19 \cdot 10^{-1}$& $ 2.38 \cdot 10^{-1} $ & $2.79 \cdot 10^{~1}$& $9.52 \cdot 10^{-1}$ & $5.02 \cdot 10^{-2}$ \\

& FLAPS (Ours) & $\best{2.28 \cdot 10^{-2}}$ & $8.90 \cdot 10^{-1}$ & $\best{3.43 \cdot 10^{~1}}$&  $\best{9.67 \cdot 10^{-1}}$ & $\best{2.42 \cdot 10^{-2}}$\\
\bottomrule
\end{tabular}}
\label{tab:reg_result_nonGP}
\end{table}

\begin{table}[h]

\caption{PDE inverse problems at resolution $128 \times 128$ from 128 noisy pointwise observations for Darcy flow, Poisson, Helmholtz, and Navier--Stokes benchmarks. Best and second-best results are highlighted in bold and brown bold, respectively.}

\footnotesize
\centering
\renewcommand{\arraystretch}{0.9}
\setlength{\tabcolsep}{3.pt} %
{
\begin{tabular}
{@{}c  !{\vrule width 0.5pt}  c !{\vrule width 0.5pt}cccccc@{}}
\toprule
Datasets & Method & CRPS $\downarrow$  & SSR $\rightarrow 1$  & PSNR $\uparrow$ & SSIM $\uparrow$ & Relative $L^2 \downarrow$ \\ 
\midrule
\multirow{5}{*}{Darcy Flow}
& DiffusionPDE &  $ \best{5.25 \cdot 10^{-2}}$& $\best{1.02 \cdot 10^{~0}}$ & $\best{2.46 \cdot 10^{~1}}$ &  $\best{7.61 \cdot 10^{-1}}$ & $\best{6.06 \cdot 10^{-1}}$   \\
& FunDPS & $ 1.24 \cdot 10^{-1}$& $1.09 \cdot 10^{~1} $ & $2.02 \cdot 10^{~1}$ &  $4.71 \cdot 10^{-1}$ & $9.92 \cdot 10^{-1}$ \\
& DDIS & $ 1.67 \cdot 10^{-1}$& $1.15 \cdot 10^{~0} $ & $1.92 \cdot 10^{~1}$ &  $2.86 \cdot 10^{-1}$ & $1.12 \cdot 10^{~0}$ \\
& FLAPS-FNO & $ 1.08 \cdot 10^{-1}$& $\sbest{1.08 \cdot 10^{~0}} $ & $\sbest{2.15 \cdot 10^{~1}}$ &  $\sbest{5.57 \cdot 10^{-1}}$ & $\sbest{8.66 \cdot 10^{-1}}$ \\  
&   FLAPS-UNet & $ \sbest{1.07 \cdot 10^{-1}}$& $1.10 \cdot 10^{~0} $ & $2.14 \cdot 10^{~1}$ &  $5.38 \cdot 10^{-1}$ & $8.74 \cdot 10^{-1}$
\\ \midrule
\multirow{5}{*}{Poisson Equation}
& DiffusionPDE &   $ 1.05 \cdot 10^{-1}$& $1.09 \cdot 10^{~0} $ & $2.55 \cdot 10^{~1}$ &  $5.29 \cdot 10^{-1}$ & $6.27 \cdot 10^{-1}$ \\
& FunDPS & $ 1.07 \cdot 10^{-1}$& $1.05 \cdot 10^{~0} $ & $2.54 \cdot 10^{~1}$ &  $5.30 \cdot 10^{-1}$ & $6.32 \cdot 10^{-1}$ \\
& DDIS & $ 9.12 \cdot 10^{-2}$& $\sbest{1.03 \cdot 10^{~0}}$ & $2.65 \cdot 10^{~1}$ &  $5.49 \cdot 10^{-1}$ & $5.62 \cdot 10^{-1}$ \\
& FLAPS-FNO & $ \sbest{8.86 \cdot 10^{-2}}$& $\best {9.80 \cdot 10^{-1}} $ & $\sbest{2.72 \cdot 10^{~1}}$ &  $\sbest{6.05 \cdot 10^{-1}}$ & $\sbest{5.15 \cdot 10^{-1}}$ \\  
&   FLAPS-UNet & $ \best{8.76 \cdot 10^{-2}}$& $ 8.80 \cdot 10^{-1}$ & $\best{2.77 \cdot 10^{~1}}$ &  $\best{6.17 \cdot 10^{-1}}$ & $\best{4.86 \cdot 10^{-1}}$\\ \midrule
\multirow{5}{*}{Helmholtz Equation}
& DiffusionPDE &  $ 9.83 \cdot 10^{-2}$& $1.06 \cdot 10^{~0} $ & $2.60 \cdot 10^{~1}$ &  $5.58 \cdot 10^{-1}$ & $5.86 \cdot 10^{-1}$  \\
& FunDPS & $ 1.63 \cdot 10^{-1}$& $1.22 \cdot 10^{~0} $ & $2.13 \cdot 10^{~1}$ &  $3.29 \cdot 10^{-1}$ & $9.93 \cdot 10^{-1}$ \\
& DDIS & $ 9.19 \cdot 10^{-2}$& $\sbest{1.04 \cdot 10^{~0}} $ & $2.61 \cdot 10^{~1}$ &  $5.56 \cdot 10^{-1}$ & $5.81 \cdot 10^{-1}$ \\
& FLAPS-FNO & $ \sbest{8.60 \cdot 10^{-2}}$& $\best{9.81 \cdot 10^{-1}}$ & $\sbest{2.73 \cdot 10^{~1}}$ &  $\sbest{5.99 \cdot 10^{-1}}$ & $\sbest{5.01 \cdot 10^{-1}}$ \\  
&   FLAPS-UNet & $ \best{8.59 \cdot 10^{-2}}$& $8.95 \cdot 10^{-1}$ & $\best{2.78 \cdot 10^{~1}}$ &  $\best{6.13 \cdot 10^{-1}}$ & $\best{4.75 \cdot 10^{-1}}$\\ \midrule
\multirow{5}{*}{Navier-Stokes (PDE)}
& DiffusionPDE &  $ 8.05 \cdot 10^{-2}$ & $1.04 \cdot 10^{~0}$ &  $2.79 \cdot 10^{~1}$ & $6.37 \cdot 10^{-1} $  & $4.21 \cdot 10^{-1}$  \\
& FunDPS & $ 9.52 \cdot 10^{-2}$& $1.07 \cdot 10^{~0} $ & $2.65 \cdot 10^{~1}$ &  $5.87 \cdot 10^{-1}$ & $4.95 \cdot 10^{-1}$ \\
& DDIS & $\best{6.02 \cdot 10^{-2}}$& $1.08 \cdot 10^{~0} $ & $\best{3.00 \cdot 10^{~1}}$ &  $\best{6.73 \cdot 10^{-1}}$ & $\best{3.30 \cdot 10^{-1}}$ \\
& FLAPS-FNO & $ 8.12 \cdot 10^{-2}$& $\sbest{9.79 \cdot 10^{-1}}$ & $2.82 \cdot 10^{~1}$ &  $6.31 \cdot 10^{-1}$ & $4.11 \cdot 10^{-1}$ \\  
&   FLAPS-UNet & $ \sbest{8.01 \cdot 10^{-2}}$& $\best{9.87 \cdot 10^{-1}}$ & $\sbest{2.86 \cdot 10^{~1}}$ &  $\sbest{6.44 \cdot 10^{-1}}$ & $\sbest{3.89 \cdot 10^{-1}}$\\
\bottomrule
\end{tabular}}
\label{tab:PDE_inverse_128_points}
\end{table}

\subsection{PDE Inverse Problems}
\label{sec:pde_inverse}

Finally, we evaluate FLAPS on PDE inverse problems, where the unknown field \(u_1\) is inferred indirectly from sparse noisy measurements of a PDE response. We consider four representative settings: Darcy flow, Poisson, Helmholtz, and Navier--Stokes inverse problems. In all cases, we use the same posterior sampler; only the observation operator changes from direct point evaluation to the composed PDE observation map \(P_\Omega \mathcal{G}_\phi\), where \(\mathcal{G}_\phi\) is a pretrained neural-operator surrogate. To evaluate different architectural choices, we report two variants: FLAPS-FNO, our default formulation leveraging an infinite-dimensional function-space neural operator prior, and FLAPS-UNet, which utilizes a standard finite-dimensional grid-based prior to demonstrate backward compatibility. Additional details on the PDE setups, baseline implementation, observation settings, and evaluation metrics are provided in ~\ref{app:baseline_implementation}.

Figure~\ref{fig:pde_inverse} visualizes representative Darcy flow and Poisson inverse problems. For each test case, we draw 32 posterior samples of the unknown input field and pass each sample through \(\mathcal{G}_\phi\) to obtain the corresponding posterior predictive solution field. From only 128 noisy solution observations on a \(128\times128\) grid, FLAPS produces input-field samples whose predicted solutions remain consistent with the sparse measurements. For Darcy flow, the samples preserve sharp coefficient-interface structures, while for the Poisson problem, they recover coherent large-scale source patterns. The output-field diagnostic plots compare noisy observed solution values with mean posterior predictive values at the same observation locations; their alignment with the diagonal indicates likelihood consistency after the sampled inputs are mapped through the surrogate forward model. For the input fields, we compute the power spectral density (PSD) over the posterior samples and plot the geometric mean together with the geometric standard-deviation band. The agreement with the ground-truth PSD shows that FLAPS captures the dominant spatial-frequency structure of the unknown input field while retaining nontrivial posterior variability.

Table~\ref{tab:PDE_inverse_128_points} reports quantitative comparisons with diffusion-based PDE posterior samplers. FLAPS achieves the best performance on most Poisson and Helmholtz metrics, remain competitive (second best) on Darcy flow and Navier--Stokes, and provides strong calibration across the benchmarks. Importantly, these results are obtained with substantially lower test-time cost. As shown in Table~\ref{tab:poisson_runtime}, FLAPS is \(1.73\times\) faster than DiffusionPDE, \(1.55\times\) faster than FunDPS, and \(2.37\times\) faster than DDIS. Overall, FLAPS achieves comparable or better posterior performance than diffusion-based baselines on most PDE inverse benchmarks while using lower computational overhead.
Finally, because the FLAPS-FNO prior is defined in function space, FLAPS supports zero-shot PDE inverse inference on query meshes not seen during prior training. Appendix~\ref{app:zero_shot_super_resolution} demonstrates this capability on a \(160\times160\) Darcy inverse problem using only 128 noisy solution observations, corresponding to approximately \(0.5\%\) of the solution grids.

\section{Conclusion}

We presented FLAPS, a function-space posterior sampling framework built on pretrained function-space flow-matching priors. By using the learned flow trajectory as an annealing path, FLAPS avoids explicit prior-density evaluation, supports varying query discretizations, and unifies direct functional regression and PDE inverse problems through a common observation-operator formulation. Across Gaussian-process, high-dimensional non-Gaussian, and PDE inverse benchmarks, FLAPS produces coherent posterior samples and achieves state-of-the-art performance on functional regression tasks. These results suggest that function-space flow-matching priors can serve as reusable Bayesian priors for scalable uncertainty-aware inference in scientific machine learning, opening new directions for posterior sampling, functional regression, and inverse modeling.

Despite these advantages, FLAPS relies on an approximate local Gaussian correction around the flow-predicted endpoint. Appendix~\ref{app:faps_error_convergence} analyzes how local transition errors propagate across annealing levels, but a full non-asymptotic convergence theory for the practical sampler remains future work. Code is available at~\url{https://github.com/yzshi5/FLAPS}.

\subsection*{Acknowledgments}
We thank Jiachen Yao and Zirui Wang for helpful discussions that contributed to the formulation of the ideas in this work. ZER is supported by a fellowship from the David and Lucile Packard Foundation. 
\clearpage

\bibliography{FAPS}

\appendix

\clearpage

\section{Proof and Algorithm for FLAPS}
\label{app:faps_proof_algorithm}

\subsection{Proof of Proposition~\ref{prop:faps_rebridge}}
\label{app:proof_faps_rebridge}

For completeness, we restate the idealized annealing property used by FLAPS.
For a finite query set \(X\), let \(\pi_1^{y,X}\) denote the endpoint posterior
over clean functions \(u_1\), and let \(q_t^X(du_t\mid u_1)\) denote the
posterior re-bridging kernel. The annealed posterior marginal at time \(t\) is
defined by
\begin{equation}
    \pi_t^{y,X}(du_t)
    =
    \int
    q_t^X(du_t\mid u_1)\,
    \pi_1^{y,X}(du_1).
    \label{eq:appendix_annealed_marginal}
\end{equation}
For posterior sampling, we use the residual-noise-free bridge with \(s(t)=1-t\),
so that \(q_1^X(du\mid u_1)=\delta_{u_1}(du)\). The residual noise
\(\sigma_{\min}\) is used only to regularize the flow-matching training path.

\noindent\textbf{Proposition~\ref{prop:faps_rebridge}
(Annealing and re-bridging, restated).}
Assume \(u_{t_k}\sim \pi_{t_k}^{y,X}\). If
\[
    u_1\sim \pi_1^{y,X}(du_1\mid u_{t_k}),
\]
and then
\[
    u_{t_{k+1}}\sim q_{t_{k+1}}^X(du_{t_{k+1}}\mid u_1),
\]
then \(u_{t_{k+1}}\sim \pi_{t_{k+1}}^{y,X}\).

\begin{proof}
Since we work on a finite query set \(X\), the bridge kernel admits a density,
which we denote by \(q_t^X(u_t\mid u_1)\). Define the marginal density of
\(u_{t_k}\) under the annealed posterior by
\begin{equation}
    m_{t_k}^X(u_{t_k})
    =
    \int
    q_{t_k}^X(u_{t_k}\mid \tilde u_1)\,
    \pi_1^{y,X}(d\tilde u_1).
    \label{eq:appendix_marginal_density}
\end{equation}
Then Eq.~\eqref{eq:appendix_annealed_marginal} can be written as
\[
    \pi_{t_k}^{y,X}(du_{t_k}) = m_{t_k}^X(u_{t_k})\,du_{t_k}.
\]
The joint law of the clean endpoint and bridge state induced by the mixture in
Eq.~\eqref{eq:appendix_annealed_marginal} is
\begin{equation}
    \pi_1^{y,X}(du_1)\,
    q_{t_k}^X(du_{t_k}\mid u_1).
    \label{eq:joint_law_appendix}
\end{equation}
Therefore, by Bayes' rule, the conditional law of the clean endpoint given the
bridge state is
\begin{equation}
    \pi_1^{y,X}(du_1\mid u_{t_k})
    =
    \frac{
        q_{t_k}^X(u_{t_k}\mid u_1)\,
        \pi_1^{y,X}(du_1)
    }{
        m_{t_k}^X(u_{t_k})
    }.
    \label{eq:clean_conditional_appendix}
\end{equation}
Multiplying both sides by the marginal law
\(\pi_{t_k}^{y,X}(du_{t_k})=m_{t_k}^X(u_{t_k})\,du_{t_k}\), we recover the joint
law:
\begin{equation}
    \pi_1^{y,X}(du_1\mid u_{t_k})\,
    \pi_{t_k}^{y,X}(du_{t_k})
    =
    \pi_1^{y,X}(du_1)\,
    q_{t_k}^X(du_{t_k}\mid u_1).
    \label{eq:conditional_joint_identity}
\end{equation}

Now let \(B\) be any measurable set. The transition described in the proposition
first samples \(u_1\) from the conditional law
\(\pi_1^{y,X}(du_1\mid u_{t_k})\), and then re-bridges to \(t_{k+1}\) using
\(q_{t_{k+1}}^X(\cdot\mid u_1)\). Hence
\begin{align}
    \mathbb P(u_{t_{k+1}}\in B)
    &=
    \iint
    q_{t_{k+1}}^X(B\mid u_1)\,
    \pi_1^{y,X}(du_1\mid u_{t_k})\,
    \pi_{t_k}^{y,X}(du_{t_k}) \\
    &=
    \iint
    q_{t_{k+1}}^X(B\mid u_1)\,
    \pi_1^{y,X}(du_1)\,
    q_{t_k}^X(du_{t_k}\mid u_1) \\
    &=
    \int
    q_{t_{k+1}}^X(B\mid u_1)\,
    \pi_1^{y,X}(du_1) \\
    &=
    \pi_{t_{k+1}}^{y,X}(B),
\end{align}
where the second equality uses
Eq.~\eqref{eq:conditional_joint_identity}, the third equality uses
\(\int q_{t_k}^X(du_{t_k}\mid u_1)=1\), and the final equality follows from
the definition of the annealed marginal in
Eq.~\eqref{eq:appendix_annealed_marginal}. Since this holds for every
measurable set \(B\), we conclude that
\(u_{t_{k+1}}\sim \pi_{t_{k+1}}^{y,X}\).
\end{proof}

\clearpage
\subsection{Flow Annealing Posterior Sampling Algorithm}
\label{app:faps_algorithm}

\begin{algorithm} 
\caption{Flow Annealing Posterior Sampling (FLAPS)}
\label{alg:faps}
\begin{algorithmic}[1]
\Require Observation \(y\), observation operator \(\mathcal A\), pretrained OFM velocity \(v_\theta\), query set \(X\), reference covariance \(\Sigma_0^X\), covariance preconditioner \(C_X\), annealing schedule \(0=t_0<\cdots<t_K=1\), bridge scale \(s(t)\),  constants $\lambda_{\min}, \lambda_{\mathrm{scale}}$;  Langevin steps \(L\), step size \(\eta\)
\State Initialize particles \(u_{t_0}^{(i)} = s(t_0)\xi_0^{(i)}\), \(\xi_0^{(i)}\sim\mathcal N(0,\Sigma_0^X)\), for \(i=1,\ldots,N\)
\For{\(k=0,\ldots,K-1\)}
    \State Compute endpoint anchors by OFM transport:
    \[
        \hat u_k^{(i)}
        =
        \Phi^\theta_{t_k\to 1}(u_{t_k}^{(i)}),
        \qquad
        \frac{d u_t}{dt}=v_\theta(t,u_t).
    \]
    \State Set \(u_1^{(i,0)}\leftarrow \hat u_k^{(i)}\) and
    \[
        \lambda_k=\max\big(\lambda_{\min},
        \lambda_{\mathrm{scale}}(1-t_k)\big).
    \]
    \For{\(\ell=0,\ldots,L-1\)}
        \State Compute likelihood gradient:
        \[
            g_{\mathrm{lik}}^{(i,\ell)}
            =
            \nabla_{u_1}\log p(y\mid u_1^{(i,\ell)})
            =
            \frac{1}{\sigma_y^2}
            J_{\mathcal A}(u_1^{(i,\ell)})^\ast
            \big(y-\mathcal A(u_1^{(i,\ell)})\big).
        \]
        \State Compute local posterior drift:
        \[
            g^{(i,\ell)}
            =
            -\frac{u_1^{(i,\ell)}-\hat u_k^{(i)}}{\lambda_k^2}
            +
            C_X g_{\mathrm{lik}}^{(i,\ell)}.
        \]
        \State Sample \(\zeta^{(i,\ell)}\sim\mathcal N(0,C_X)\).
        \State Update by covariance-preconditioned Langevin dynamics:
        \[
            u_1^{(i,\ell+1)}
            =
            u_1^{(i,\ell)}
            +
            \eta g^{(i,\ell)}
            +
            \sqrt{2\eta}\,\zeta^{(i,\ell)}.
        \]
    \EndFor
    \State Set \(u_k^{(i),\mathrm{clean}}\leftarrow u_1^{(i,L)}\).
    \If{\(k<K-1\)}
        \State Re-bridge to the next annealing level:
        \[
            u_{t_{k+1}}^{(i)}
            =
            t_{k+1}u_k^{(i),\mathrm{clean}}
            +
            s(t_{k+1})\xi_{k+1}^{(i)},
            \qquad
            \xi_{k+1}^{(i)}\sim\mathcal N(0,\Sigma_0^X).
        \]
    \Else
        \State Set \(u_{t_K}^{(i)}\leftarrow u_k^{(i),\mathrm{clean}}\).
    \EndIf
\EndFor
\State \Return \(\{u_{t_K}^{(i)}\}_{i=1}^N\) as posterior samples from \(p(u_1\mid y)\).
\end{algorithmic}
\end{algorithm}

\clearpage

\subsection{Error propagation of practical FLAPS}
\label{app:faps_error_convergence}

Proposition~\ref{prop:faps_rebridge} shows that the ideal FLAPS transition is
exact when the clean conditional \(\pi_1^{y,X}(du_1\mid u_{t_k})\) can be
sampled exactly. In practice, this conditional is replaced by an approximate
local correction around the OFM endpoint anchor, followed by finite-step
Langevin dynamics and re-bridging. The result below should be interpreted as an
error-propagation statement rather than a standalone convergence theorem: it
shows how local one-step transition errors accumulate across annealing levels.

For \(k=0,\ldots,K-1\), define the ideal transition kernel
\begin{equation}
    K_k(u_{t_k},du_{t_{k+1}})
    =
    \int
    q_{t_{k+1}}^X(du_{t_{k+1}}\mid u_1)\,
    \pi_1^{y,X}(du_1\mid u_{t_k}).
    \label{eq:ideal_kernel_convergence}
\end{equation}
By Proposition~\ref{prop:faps_rebridge},
\begin{equation}
    \pi_{t_k}^{y,X}K_k=\pi_{t_{k+1}}^{y,X}.
\end{equation}
Let \(\widehat K_k\) denote the practical FLAPS transition kernel, including
OFM endpoint transport, local Gaussian correction, finite-step Langevin
dynamics, and re-bridging. Let \(\widehat\pi_{t_k}^{y,X}\) denote the law of
the practical particles at level \(t_k\).

\begin{lemma}[Error accumulation under approximate transitions]
\label{lemma:faps_error_convergence}
Assume that, for each \(k\),
\begin{equation}
    \sup_{u_{t_k}}
    d_{\mathrm{TV}}
    \left(
        \widehat K_k(u_{t_k},\cdot),
        K_k(u_{t_k},\cdot)
    \right)
    \leq
    \varepsilon_k .
    \label{eq:faps_step_error}
\end{equation}
Then
\begin{equation}
    d_{\mathrm{TV}}
    \left(
        \widehat\pi_{t_K}^{y,X},
        \pi_{t_K}^{y,X}
    \right)
    \leq
    d_{\mathrm{TV}}
    \left(
        \widehat\pi_{t_0}^{y,X},
        \pi_{t_0}^{y,X}
    \right)
    +
    \sum_{k=0}^{K-1}\varepsilon_k .
    \label{eq:faps_error_accumulation}
\end{equation}
\end{lemma}

\begin{proof}
Using the triangle inequality and contraction of Markov kernels in total
variation,
\begin{align}
    d_{\mathrm{TV}}
    \left(
        \widehat\pi_{t_{k+1}}^{y,X},
        \pi_{t_{k+1}}^{y,X}
    \right)
    &=
    d_{\mathrm{TV}}
    \left(
        \widehat\pi_{t_k}^{y,X}\widehat K_k,
        \pi_{t_k}^{y,X}K_k
    \right) \nonumber \\
    &\leq
    d_{\mathrm{TV}}
    \left(
        \widehat\pi_{t_k}^{y,X}\widehat K_k,
        \widehat\pi_{t_k}^{y,X}K_k
    \right)
    +
    d_{\mathrm{TV}}
    \left(
        \widehat\pi_{t_k}^{y,X}K_k,
        \pi_{t_k}^{y,X}K_k
    \right) \nonumber \\
    &\leq
    \varepsilon_k
    +
    d_{\mathrm{TV}}
    \left(
        \widehat\pi_{t_k}^{y,X},
        \pi_{t_k}^{y,X}
    \right).
\end{align}
Iterating over \(k=0,\ldots,K-1\) gives
Eq.~\eqref{eq:faps_error_accumulation}.
\end{proof}

Lemma~\ref{lemma:faps_error_convergence} does not by itself prove that the
practical transition is consistent; rather, it shows that if the practical
transition approximates the ideal transition at each annealing level, then the
global sampling error grows at most additively with the local transition errors.
The one-step error \(\varepsilon_k\) contains several sources of approximation:
the learned OFM velocity-field error, numerical ODE discretization error in the
endpoint transport, the local Gaussian approximation around the endpoint anchor,
finite-step Langevin discretization and mixing error, covariance-estimation
error, and any approximation in the re-bridging kernel. Under additional regularity and local-consistency assumptions, one may expect a
decomposition of the form
\begin{equation}
    \varepsilon_k
    \lesssim
    C_{\mathrm{loc}}\Delta_k^{1+\alpha}
    +
    C_{\mathrm{flow}}e_{\theta}
    +
    C_{\mathrm{ODE}}h_{\mathrm{ODE}}^{p}
    +
    a_k(\eta,L)
    +
    e_{\mathrm{cov}},
    \label{eq:faps_step_error_decomposition}
\end{equation}
where \(\Delta_k=t_{k+1}-t_k\), \(e_{\theta}\) denotes the learned velocity-field
approximation error, \(h_{\mathrm{ODE}}\) is the ODE solver step size, \(p\) is
the order of the ODE solver, \(a_k(\eta,L)\) denotes the finite-step Langevin
error, and \(e_{\mathrm{cov}}\) denotes the covariance-preconditioner estimation
error. Rather than serving as an explicit theorem, this decomposition  provides a clear conceptual roadmap, identifying the precise approximation terms 
that must be controlled to construct a full non-asymptotic convergence theory .

If these local errors vanish ($\sum_{k=0}^{K-1}\varepsilon_k \to 0$) as the annealing mesh, prior transport, and Langevin parameters are refined, Eq.~\eqref{eq:faps_error_accumulation} guarantees convergence to the exact target posterior. While proving formal non-asymptotic bounds for neural-operator flow priors and nonlinear PDEs remains an open direction, Lemma~\ref{lemma:faps_error_convergence} establishes a stability guarantee that is validated empirically through our posterior metrics and ablation studies.

\section{Experimental Setup}
\label{sec:experimental_setup}

We evaluate FLAPS on two classes of tasks: direct functional (stochastic process) regression and PDE inverse problems. For functional regression, we use controlled Gaussian-process benchmarks and high-dimensional non-Gaussian fields following~\citep{shi_stochastic_2026,shi_mesh-informed_2025}. For PDE inverse problems, we use the FunDPS benchmark datasets~\citep{yao_guided_2026}. All neural-process baselines are implemented following~\citep{abu_hamad_flow_2026}. Dataset statistics are summarized in Table~\ref{tab:app_dataset}; observation settings, noise levels (variance), posterior sample counts, and query resolutions are reported in Table~\ref{tab:app_obs}. All runtimes reported in the following tables are measured from a single run on one NVIDIA RTX 6000 Ada GPU with 48 GB memory. For all cases except the global climate, the reference Gaussian measure \(\gamma\) used for training the OFM prior is specified by a Matérn kernel with smoothness parameter \(\zeta=0.5\) and length scale \(l=0.01\) (for the global climate case, \(l=0.05\)).

\subsection{Datasets for Functional Regression}
\label{app:functional_regression_datasets}

\textbf{Matérn-kernel GP.}
We use a Matérn Gaussian process on the domain \([0,1]\) with length scale
\(l=0.3\) and smoothness parameter \(\zeta=1.5\). We generate \(20{,}000\)
training samples at a fixed resolution of \(128\). Both the OFM prior and the
neural-process baselines are trained at this low resolution. We evaluate on
\(100\) test functions at query resolutions \(128\) and \(512\). For each test
case, we randomly select \(7\) observed locations and add Gaussian noise with
variance \(10^{-2}\).

\textbf{Nonstationary Gibbs-kernel GP.}
For the Gibbs kernel, we use an input-dependent length scale
\[
    \ell(x)=\ell_0+\ell_1 x,\qquad x\in[0,1],
\]
which induces the covariance
\[
    k(x,x')
    =
    \sigma^2
    \sqrt{
    \frac{2\ell(x)\ell(x')}
    {\ell(x)^2+\ell(x')^2}
    }
    \exp\left(
    -\frac{(x-x')^2}
    {\ell(x)^2+\ell(x')^2}
    \right).
\]
We set \(\ell_0=0.05\), \(\ell_1=0.25\), and \(\sigma=1.0\). All other
settings are identical to the Matérn-kernel GP experiment.

\textbf{Navier--Stokes.}
This dataset consists of solutions to the two-dimensional Navier--Stokes
equations on a torus at resolution \(64\times64\)~\citep{li_fourier_2021}. For
each test field, we observe \(64\) randomly selected spatial locations and add
Gaussian noise with variance \(10^{-2}\).

\textbf{Black hole.}
We use the black-hole imaging dataset from~\citep{shi_stochastic_2026}. The
training set contains \(11{,}600\) images at resolution \(64\times64\), after
rotation-based data augmentation. For each test image, we observe \(256\)
randomly selected pixels and add Gaussian noise with variance \(10^{-3}\).

\textbf{Global climate.}
We use the real-world global climate dataset from~\citet{dupont_generative_2022},
which contains global temperature measurements over the past \(40\) years. Each
sample is a function defined on a \(46\times90\) latitude--longitude grid.
Following~\citet{dupont_generative_2022}, we convert latitude--longitude
coordinates to Euclidean coordinates in \(\mathbb R^3\) before passing them to
the models. The dataset contains \(9{,}676\) training samples. For each test
case, we observe \(128\) randomly selected spatial locations and add Gaussian
noise with variance \(10^{-3}\).

\subsection{Datasets for PDE Inverse Problems}
\label{app:pde_inverse_datasets}

We consider four PDE inverse-problem benchmarks from DiffusionPDE/FunDPS~\citep{huang_diffusionpde_2024,yao_guided_2026}: Darcy flow, Helmholtz, non-bounded Navier--Stokes, and the Poisson equation. We use the normalized datasets provided by the authors and refer readers to~\citep{yao_guided_2026} for dataset details. Each dataset is defined on a \(128\times128\) grid and consists of paired input and solution fields. Each benchmark contains \(50{,}000\) training samples. To keep the PDE descriptions readable, we use \(u\) to denote the unknown input field in the PDE-specific equations below. In the unified posterior-sampling notation used in the main text, this same unknown field is denoted by \(u_1\). Thus, throughout this subsection, \(u\) and \(u_1\) refer to the same input field, while \(w=\mathcal G(u)\) denotes the corresponding PDE solution field. This notation differs from the original benchmark descriptions, where the input and solution fields are often denoted by \(a(x)\) and \(u(x)\), respectively.

For each PDE inverse problem, the goal is to infer the unknown input field \(u\) from sparse observations of the solution field \(w\). The forward operator \(\mathcal G\) is approximated by a pretrained FNO surrogate \(\mathcal G_\phi\), which maps input fields to solution fields. The function-space flow-matching prior is trained on the distribution of input fields. At test time, we randomly select \(128\) solution-observation locations for each test case and add Gaussian noise with variance \(10^{-3}\). For consistency with the main posterior formulation, we write the observation model as
\[
    y = P_\Omega \mathcal G_\phi(u_1) + \epsilon,
    \qquad
    \epsilon \sim \mathcal N(0,10^{-3}I),
\]
where \(u_1\equiv u\) is the unknown input field and \(P_\Omega\) denotes the sparse observation operator. We emphasize that, unlike prior studies~\citep{yao_guided_2026,huang_diffusionpde_2024, lin_decoupled_2026}, which evaluate on noise-free point observations, our setting explicitly corrupts the observed solution values with Gaussian noise. All methods are given only these noisy observations. The dataset descriptions below restate the benchmark definitions from~\citep{yao_guided_2026,huang_diffusionpde_2024}, with notation adapted to match our unified functional-regression and inverse-problem formulation.

\paragraph{Darcy flow.}
We consider the Darcy flow equation on the unit square,
\begin{equation}
    -\nabla \cdot \big(u(x)\nabla w(x)\big) = f(x),
    \qquad x \in (0,1)^2,
\end{equation}
with unit forcing \(f(x)=1\) and zero boundary conditions. The coefficient field is sampled as $ u \sim h_{\#}\mathcal{N}\big(0,(-\Delta+9I)^{-2}\big),$
where \(h:\mathbb{R}\to\mathbb{R}\) thresholds the Gaussian field by setting \(h(z)=12\) if \(z>0\) and \(h(z)=3\) otherwise.

\paragraph{Poisson equation.}
We consider the Poisson equation on the unit square,
\begin{equation}
    \nabla^2 w(x) = u(x),
    \qquad x\in(0,1)^2,
\end{equation}
with homogeneous Dirichlet boundary conditions \(w|_{\partial\Omega}=0\). The source field is sampled from a Gaussian random field,
$u \sim \mathcal{N}\big(0,(-\Delta+9I)^{-2}\big).$

\paragraph{Helmholtz equation.}
We consider the Helmholtz equation on the unit square,
\begin{equation}
    \nabla^2 w(x) + k^2 w(x) = u(x),
    \qquad x \in (0,1)^2,
\end{equation}
with \(k=1\) and homogeneous Dirichlet boundary conditions \(w|_{\partial\Omega}=0\). The coefficient field \(u(x)\) is sampled from a Gaussian random field following~\citep{li_fourier_2021}.

\paragraph{Navier--Stokes equations.}
We consider the two-dimensional incompressible Navier--Stokes equations in
vorticity form on the unit square. Let \(v(x,t)\) denote the velocity field and
let \(\zeta(x,t)=\nabla\times v(x,t)\) denote the vorticity. The dynamics are
given by
\begin{align}
    \partial_t \zeta(x,t) + v(x,t)\cdot \nabla \zeta(x,t)
    &= \nu \Delta \zeta(x,t) + f(x),
    \qquad x\in(0,1)^2,\; t\in(0,T], \\
    \nabla \cdot v(x,t) &= 0,
    \qquad x\in(0,1)^2,\; t\in[0,T], \\
    v(x,0) &= u(x),
    \qquad x\in(0,1)^2 .
\end{align}
Here \(\nu=10^{-3}\) is the viscosity and \(T=1\), $w(x) = v(x, T)$. The initial vorticity field
is sampled as $
    u \sim \mathcal{N}\!\left(
    0,\, 7^{3/2}(-\Delta+49I)^{-5/2}
    \right),
$
and the forcing term is fixed as
$
    f(x)=\frac{1}{10}\left[
    \sin\!\left(2\pi(x_1+x_2)\right)
    + \cos\!\left(2\pi(x_1+x_2)\right)
    \right].
$
The PDE is solved using a pseudo-spectral method following~\citep{li_fourier_2021}.

\begin{table}[ht]
\renewcommand{\arraystretch}{0.9}

  \centering
  \footnotesize 

\caption{Summary of datasets used for functional regression and PDE inverse experiments.}
  \label{tab:benchmarks}
  \begin{tabular}{c c c c c c}
    \toprule
    Problem & Datasets & Training Resolution & 
    Training Samples
    & Test Samples\\
    \midrule
    \multirow{3}{*}{Functional Regression}
      & Matérn-Kernel GP & 128 & 20,000 & 100 \\
      & Gibbs-Kernel GP    & 256    & 20,000     & 100 \\
      & Navier-Stokes (Regression)   & $64\times64$      & 30,000    & 100  \\
      & Black Hole    & $64\times64$        & 11,600    & 100  \\
      & Global Climate    & (manifold) mesh=4,140    & 9,676    & 100  \\
    \midrule
    \multirow{3}{*}{PDE Inverse}
      & Darcy Flow & $128\times128$ & 50,000 & 100\\
      & Helmholtz Equation         & $128\times128$       &50,000   & 100 \\
      & Navier-Stokes (PDE)         & $128\times128$        & 50,000  & 100  \\
      & Poisson Equation         & $128\times128$        & 50,000  & 100  \\
    \bottomrule
  \end{tabular}
  \label{tab:app_dataset}
\end{table}

\begin{table}[ht]

  \centering

\caption{Noise levels (variance), observation counts, posterior samples per test case, and query resolutions used in each benchmark.}
  \label{tab:benchmarks}
  \scriptsize
  \begin{tabular}{c c c c c c}
    \toprule
    Problem & Datasets & 
    Noise Level & Num. of Obs
    & Post. Samples  & Query Res.\\
    \midrule
    \multirow{5}{*}{Func. Regression}
      & Matérn-Kernel GP & 1e-2 & 7 & 128 & 128~/~512~/~1024 \\
      & Gibbs-Kernel GP    & 1e-2   & 7     & 128 & 128~/~512~/~1024 \\
      & Navier-Stokes (Regression)    & 1e-2    & 64    & 32  & $64\times64$\\
      & Black Hole    & 1e-3       & 256    & 32  & $64\times64$\\
      & Global Climate    & 1e-3 & 128    & 32    & 4,140  \\
    \midrule
    \multirow{4}{*}{PDE Inverse}
      & Darcy Flow & 1e-3 & 128 & 32 & $128\times 128$ / $160\times 160$\\
      & Helmholtz Equation         & 1e-3       &128  & 32 &$128\times 128$ \\
      & Navier-Stokes (PDE)         & 1e-3        & 128 & 32 &$128\times 128$  \\
      & Poisson Equation         & 1e-3       &128  & 32 &$128\times 128$  \\
    \bottomrule
  \end{tabular}
  \label{tab:app_obs}
\end{table}

\subsection{Evaluation Metrics}
\label{app:evaluation_metrics}

For one-dimensional GP regression, exact reference posterior samples are
available. We therefore evaluate posterior sample quality using Sliced
Wasserstein Distance (SWD) and Maximum Mean Discrepancy (MMD), both of which
directly compare generated posterior samples with reference posterior samples.
Lower values indicate better agreement with the reference posterior.

For high-dimensional functional regression and noisy PDE inverse problems, exact
posterior distributions are unavailable. Following~\citet{zheng_blade_2026}, we
therefore report both probabilistic and reconstruction metrics. Continuous
Ranked Probability Score (CRPS) evaluates the quality of the predictive
posterior, with lower values indicating better performance. Spread--Skill Ratio
(SSR) measures ensemble calibration, with an ideal value close to \(1\). Peak
Signal-to-Noise Ratio (PSNR) and Structural Similarity Index (SSIM) measure
reconstruction fidelity, with higher values indicating better reconstructions.
Relative \(L^2\) error measures normalized reconstruction error.

For PDE inverse problems, we emphasize that reconstruction metrics alone do not fully
characterize posterior quality in sparse and noisy inverse problems, where multiple input
fields may be consistent with the same noisy observations. The metric definitions below follow those used in
\citep{shi_mesh-informed_2025,zheng_blade_2026}, with notation adapted to our
functional-regression and inverse-problem setting. The reported metrics are averaged over all test cases.

\paragraph{Sliced Wasserstein distance.}
We measure the discrepancy between generated samples \(P\) and reference
posterior samples \(Q\) using the sliced Wasserstein distance:
\[
    \mathrm{SWD}_p(P,Q)
    \approx
    \left(
    \frac{1}{L}\sum_{\ell=1}^{L}
    W_p^p\big(
        \langle P,\theta_\ell\rangle,
        \langle Q,\theta_\ell\rangle
    )
    \right)^{1/p},
\]
where \(\theta_\ell\) denotes a random projection direction,
\(\langle P,\theta_\ell\rangle\) and \(\langle Q,\theta_\ell\rangle\) are the
corresponding empirical one-dimensional projected distributions, and \(W_p\) is
the one-dimensional Wasserstein distance of order \(p\). In our experiments, we
set \(p=2\) and report the averaged SWD following the implementation of
\citet{shi_mesh-informed_2025}.

\paragraph{Maximum mean discrepancy.}
We evaluate the discrepancy between generated samples and reference posterior samples using the maximum mean discrepancy (MMD). Given generated samples \(P=\{h_{1}^i\}_{i=1}^{n}\) and reference samples \(Q=\{h_{2}^j\}_{j=1}^{m}\), the squared MMD is estimated as
\[
    \mathrm{MMD}^2(P,Q)
    =
    \frac{1}{n(n-1)}\sum_{i\neq i'} k(h_{1}^i,h_{1}^{i'})
    +
    \frac{1}{m(m-1)}\sum_{j\neq j'} k(h_{2}^j,h_{2}^{j'})
    -
    \frac{2}{nm}\sum_{i=1}^{n}\sum_{j=1}^{m} k(h_1^i,h_2^j),
\]
where \(k(\cdot,\cdot)\) is a positive-definite kernel. In our experiments, we use the Gaussian RBF kernel
\[
    k(h_1,h_2)=\exp\left(-\frac{\|h_1-h_2\|_2^2}{2\sigma^2}\right),
\]
with bandwidth \(\sigma\). Lower MMD values indicate that the generated samples better match the reference posterior distribution.

\paragraph{Relative \(L^2\) error.}
Given one prediction \(h\) and the ground truth \(h^\ast\), we report the relative \(L^2\) error:
\[
    \mathrm{Rel}\text{-}L^2(h,h^\ast)
    =
    \frac{\|h-h^\ast\|_2}{\|h^\ast\|_2}.
\]

\paragraph{Continuous ranked probability score (CRPS).}
We use the continuous ranked probability score (CRPS) to assess the quality of the posterior predictive distribution~\citep{gneiting_strictly_2007}. For a predictive random variable \(h\) and a ground-truth observation \(h^\ast\), CRPS is defined as
\[
    \mathrm{CRPS}
    =
    \mathbb{E}|h-h^\ast|
    -
    \frac{1}{2}\mathbb{E}|h-h'|,
\]
where \(h\) and \(h'\) are independent samples from the predictive distribution, and \(h^\ast\) denotes the observed ground truth. The first term measures the discrepancy between posterior samples and the observation, while the second term rewards appropriate ensemble spread. As a proper scoring rule, CRPS is minimized in expectation when the predictive distribution matches the data-generating distribution. Lower CRPS therefore indicates posterior samples that are more accurate and better calibrated.

\paragraph{Spread--skill ratio (SSR).}
We use the spread--skill ratio (SSR) to assess the calibration of posterior samples~\citep{fortin_why_2014}. 
Given \(N\) test cases with ground truth \(h_i^\ast\) and ensemble predictions \(\{h_{i,j}\}_{j=1}^{J}\), let $
    \bar{h}_i = \frac{1}{J}\sum_{j=1}^{J} h_{i,j}.
$
The SSR is defined as
\[
    \mathrm{SSR}
    =
    \sqrt{\frac{\mathrm{spread}^2}{\mathrm{skill}^2}},
\]
where
\[
    \mathrm{spread}^2
    =
    \frac{1}{N}\sum_{i=1}^{N}
    \frac{1}{J-1}\sum_{j=1}^{J}
    \|h_{i,j}-\bar{h}_i\|_2^2,
\]
and
\[
    \mathrm{skill}^2
    =
    \frac{1}{N}\sum_{i=1}^{N}
    \|\bar{h}_i-h_i^\ast\|_2^2 .
\]
An ideal calibrated ensemble has \(\mathrm{SSR}\approx 1\). Values below one indicate under-dispersion, or over-confidence, while values above one indicate over-dispersion, or overly conservative uncertainty estimates.

\paragraph{Peak signal-to-noise ratio (PSNR).}
We use the peak signal-to-noise ratio (PSNR) to evaluate reconstruction quality. Given a prediction \(h\) and the ground truth \(h^\ast\), PSNR is defined as
\[
    \mathrm{PSNR}(h,h^\ast)
    =
    20\log_{10}(\mathrm{MAX})
    -
    10\log_{10}\big(\mathrm{MSE}(h,h^\ast)\big),
\]
where \(\mathrm{MAX}\) denotes the maximum possible signal value, or equivalently the prescribed data range for normalized fields. Higher PSNR indicates better reconstruction quality. Since the empirical data range varies across datasets, we use a fixed data range for consistency and simplicity: \(\mathrm{MAX}=1.0\) for functional regression tasks and \(\mathrm{MAX}=5.0\) for PDE tasks.

\paragraph{Structural similarity index measure (SSIM).}
We use the structural similarity index measure (SSIM) to evaluate perceptual similarity between a prediction \(h\) and ground truth \(h^\ast\)~\citep{wang_image_2004}. SSIM is defined as
\[
    \mathrm{SSIM}(h,h^\ast)
    =
    \frac{(2\mu_h\mu_{h^\ast}+C_1)(2\sigma_{h h^\ast}+C_2)}
    {(\mu_h^2+\mu_{h^\ast}^2+C_1)(\sigma_h^2+\sigma_{h^\ast}^2+C_2)},
\]
where \(\mu_h\) and \(\mu_{h^\ast}\) are the mean intensities, \(\sigma_h^2\) and \(\sigma_{h^\ast}^2\) are the variances, \(\sigma_{h h^\ast}\) is the covariance, and \(C_1,C_2\) are stabilization constants (with slight abuse of notation). Higher SSIM indicates stronger structural similarity.

\subsection{Experimental Configurations}
\label{app:experimental_configurations}

All methods are evaluated on shared test sets with the same random observation masks and noise realizations across baselines. FLAPS uses pretrained OFM priors and performs posterior inference at test time without retraining for new observation sets. The baseline ``OFM'' refers to the OFM posterior sampling algorithm~\citep{shi_stochastic_2026}, which uses the same pretrained OFM prior but does not include the FLAPS annealed correction and re-bridging procedure.

For GP regression, the OFM prior is trained for \(500\) epochs. For PDE inverse problems, the input-field priors are trained for \(100\) epochs using either an FNO or UNet backbone. The FNO prior corresponds to the function-space OFM setting, while the UNet prior transports finite-dimensional white noise to the data distribution. This UNet variant demonstrates the backward compatibility of FLAPS: beyond function-space OFM priors, FLAPS can also be directly applied to standard finite-dimensional flow-matching priors for inverse problems. The PDE forward operators are approximated by pretrained FNO surrogates, which are kept frozen throughout posterior sampling.

Unless otherwise specified, FLAPS draws \(32\) posterior samples per test case, and all reported metrics are averaged over the corresponding test set. 
For functional regression, we use \(40\) annealing steps, \(20\) ODE steps for endpoint transport, \(50\) Langevin correction steps per annealing level, a Langevin step size of \(10^{-3}\), and a low-rank covariance preconditioner with rank \(32\). 
Additional prior architectures and posterior-sampling hyperparameters are provided in Tables~\ref{tab:ofm_prior_architectures} and~\ref{tab:posterior_sampling_hyperparams}.

For PDE inverse problems, we report two FLAPS variants. 
FLAPS-FNO is the standard function-space setting, where the prior is parameterized by an FNO and the reference distribution \(\mathcal{N}(0,\Sigma_0^{X})\) in Algorithm~\ref{alg:faps} is the finite-dimensional marginal of the Gaussian reference process on the query set \(X\). 
FLAPS-UNet uses a standard finite-dimensional flow-matching prior, where a UNet transports white noise to the data distribution. 
Both variants use the same posterior-sampling procedure; they differ only in the prior backbone and the corresponding reference distribution. 
This demonstrates the backward compatibility of FLAPS with finite-dimensional flow-matching priors. 
However, only FLAPS-FNO naturally supports zero-shot super-resolution inverse problems, since the function-space prior can be evaluated on unseen query resolutions; see Appendix~\ref{app:zero_shot_super_resolution}. For PDE inverse experiments, we use \(20\) annealing steps, \(10\) ODE steps for endpoint transport, \(40\) Langevin correction steps per annealing level, a Langevin step size of \(4 \times 10^{-5}\), and a low-rank covariance preconditioner with rank \(32\). 
Additional PDE prior architectures and inverse-sampling hyperparameters are provided in Tables~\ref{tab:app_OFM_PDE} and~\ref{tab:posterior_sampling_inverse_hyperparams}.

\newpage
\begin{table}[ht]
\renewcommand{\arraystretch}{0.9}

  \centering
  \footnotesize 

    \caption{Architecture and model size of pretrained OFM priors for functional regression.}
  \label{tab:benchmarks}
  \begin{tabular}{c c c c c c}
    \toprule
     Func. Regression & Datasets & Architecture & 
    Modes/ Hidden channel/ Layers
    & Num. of Parameters \\
    \midrule
    \multirow{5}{*}{OFM prior}
      & Matérn-Kernel GP & 1D FNO & 32/256/4 & 5.25 M \\
      & Gibbs-Kernel GP    & 1D FNO    & 32/256/4    & 5.25 M \\
      & Navier-Stokes    & 2D FNO      & 24/128/4    & 20.6 M  \\
      & Black Hole    & 2D FNO        & 24/128/4   & 20.6 M  \\
      & Global Climate    & MINO   & /    & 21.4 M  \\
    \bottomrule
  \end{tabular}
    \label{tab:ofm_prior_architectures}
\end{table}

\begin{table}[ht]
\centering
\caption{FLAPS hyperparameters for functional regression}
\label{tab:posterior_sampling_hyperparams}
\begin{tabular}{clc}
\toprule
\textbf{Algorithm} & \textbf{Items} & \textbf{Values} \\
\midrule
\multirow{9}{*}{\shortstack[l]{FLAPS}}
& Posterior samples per case & 128 (GP)~/~32 (non-GP) \\
& Annealing steps & 40 \\
& ODE steps (anchor transport) & 20 \\
& $\lambda_{\min} / \lambda_{\mathrm{scale}}$ & 0.05/ 1 \\
& Langevin steps / level & 50 \\
& Langevin learning rate & \(10^{-3}\) (GP) / \(10^{-4}\) (non-GP)  \\
& Low-rank covariance rank & 32 \\
& Low-rank covariance samples & 256 \\
& Low-rank covariance ODE steps & 20 \\
\bottomrule
\end{tabular}
\end{table}

\begin{table}[ht]
\renewcommand{\arraystretch}{0.9}

  \centering
  \footnotesize 

\caption{Architecture and model size of OFM, UNet priors and PDE surrogates for inverse problems.}
  \label{tab:benchmarks}
  \begin{tabular}{c c c c c c}
    \toprule
     PDE Inverse & Datasets & Architecture & 
    Modes/ Hidden channel/ Layers
    & Num. of Parameters \\
    \midrule
    \multirow{4}{*}{OFM prior}
      & Darcy Flow & 2D FNO & 48/64/4 & 19.7 M \\
      & Helmholtz Equation   & 2D FNO     & 48/64/4 & 19.7 M\\
      & Navier Stokes (PDE)    & 2D FNO       & 48/64/4 & 19.7 M  \\
      & Poisson Equation    & 2D FNO         & 48/64/4 & 19.7 M  \\
    \midrule
    \multirow{4}{*}{UNet prior}
      & Darcy Flow & Diffusion UNet & - & 14.9 M \\
      & Helmholtz Equation   & Diffusion UNet     & - & 14.9 M\\
      & Navier Stokes (PDE)    & Diffusion UNet      & - & 14.9 M  \\
      & Poisson Equation    & Diffusion UNet         & - & 14.9 M  \\
    \midrule
    \multirow{4}{*}{PDE surrogate}
      & Darcy Flow & 2D FNO & 48/64/4 & 19.7 M \\
      & Helmholtz Equation   & 2D FNO     & 48/64/4 & 19.7 M\\
      & Navier Stokes (PDE)  & 2D FNO       & 48/64/4 & 19.7 M  \\
      & Poisson Equation     & 2D FNO         & 48/64/4 & 19.7 M  \\
      \bottomrule
    
  \end{tabular}
  \label{tab:app_OFM_PDE}
\end{table}

\begin{table}[h]
\centering
\caption{FLAPS hyperparameters for noisy PDE invese problem}
\label{tab:posterior_sampling_inverse_hyperparams}
\begin{tabular}{clc}
\toprule
\textbf{Algorithm} & \textbf{Items} & \textbf{Values} \\
\midrule
\multirow{9}{*}{\shortstack[l]{FLAPS}}
& Posterior samples per case & 32 \\
& Annealing steps & 20 \\
& ODE steps (anchor transport) & 10 \\
& $\lambda_{\min} / \lambda_{\mathrm{scale}}$ & 0.05/ 1 \\
& Langevin steps / level & 40 \\
& Langevin learning rate & \(4\times10^{-5}\) \\
& Low-rank covariance rank & 32 \\
& Low-rank covariance samples & 256 \\
& Low-rank covariance ODE steps & 20 \\
\bottomrule
\end{tabular}
\end{table}

\subsection{Baseline Implementation Details}
\label{app:baseline_implementation}

We describe the implementation details for the baselines used in our experiments. 
For neural-process baselines, we follow the standard FlowNP implementation~\citep{abu_hamad_flow_2026} available at \url{https://github.com/danrsm/flowNP}. For Global climate, we remove the comparison with OFM posterior sampling baseline, since combined with MINO prior, OFM posterior sampling is extremely slow and we also encountered numerically instability in this case. 
All methods are evaluated on shared test sets with identical observation masks and noise realizations.

For PDE inverse problems, we compare FLAPS with DiffusionPDE, FunDPS, and DDIS. 
For FLAPS, we consider two prior backbones: an FNO-based function-space prior, denoted FLAPS-FNO, and a UNet-based finite-dimensional flow-matching prior, denoted FLAPS-UNet. 
FLAPS-FNO is our default function-space setting, while FLAPS-UNet demonstrates the backward compatibility of FLAPS with standard finite-dimensional flow-matching priors. 
Architectural details are provided in Table~\ref{tab:app_OFM_PDE}.

To ensure a fair comparison, we match prior backbones and model sizes across methods whenever possible. 
DiffusionPDE, DDIS, and FLAPS-UNet use the same UNet backbone with the same model size for prior learning, while FunDPS and FLAPS-FNO use the same FNO backbone with the same model size. 
For FunDPS and DiffusionPDE, we start from the official recommended guidance-strength parameters and make only minor adjustments for our observation setting. For DDIS, we choose official recommended parameters for annealing, prior and Langevin step (100, 5, 20).
The PDE forward operators used by FLAPS and DDIS are approximated by pretrained FNO surrogates, which are kept fixed during posterior sampling.

For all methods, we exclude additional PDE residual guidance or PDE-specific training losses beyond the observation likelihood used for posterior sampling. 
We also do not apply the multi-resolution sampling strategy proposed by FunDPS, so that all baselines are compared under a consistent single-resolution inference setting For DDIS, we implement Algorithm~1, the DDIS-DAPS sampler, from~\citet{lin_decoupled_2026}. 
Although the DDIS appendix reports additional experimental settings and weighting coefficients, some of these details are not well explained in that paper and not fully specified for our benchmark configuration. 
Our preliminary attempts to incorporate these extra settings (e.g. RBF noise injection and include additional weights for Langevin steps) led to worse performance. 
Therefore, we report results using the default DDIS-DAPS sampler described in the main algorithm, together with the same pretrained FNO forward surrogate used by FLAPS.

Table~\ref{tab:poisson_runtime} compares test-time computational cost on the Poisson inverse benchmark. 
Here, ``prior sampling steps'' denote the number of diffusion steps used by the DiffusionPDE prior. 
For FLAPS and DDIS, the sampling configuration is reported as 
\((\text{annealing}, \text{prior transport/sampling}, \text{Langevin})\) steps. 
Each method draws \(32\) posterior samples per test case. 
FLAPS requires substantially less computation: FLAPS-FNO and FLAPS-UNet take \(65.1\)s and \(64.9\)s per test case, respectively, compared with \(112.5\)s for DiffusionPDE, \(100.7\)s for FunDPS, and \(153.9\)s for DDIS. 
Using FLAPS-UNet as the reference, this corresponds to a \(1.73\times\) speedup over DiffusionPDE, a \(1.55\times\) speedup over FunDPS, and a \(2.37\times\) speedup over DDIS under the same posterior-sample budget.

\begin{table}[h]
\centering

\caption{
Test-time computational cost on the Poisson inverse benchmark. 
Speedup is reported relative to FLAPS-UNet. 
The sampling configuration denotes the number of annealing, prior transport/sampling, and Langevin correction steps, respectively.
}

\label{tab:poisson_runtime}

\begin{tabular}{lccc}
\toprule
Method 
& Prior backbone / size
& PDE surrogate / size
& Prior sampling steps \\
\midrule
DiffusionPDE  & UNet (14.9M) & --          & 1000 \\
FunDPS        & FNO (19.7M)  & --          & 1000 \\
DDIS          & UNet (14.9M) & FNO (19.7M) & -- \\
FLAPS-FNO      & FNO (19.7M)  & FNO (19.7M) & -- \\
FLAPS-UNet     & UNet (14.9M) & FNO (19.7M) & -- \\
\midrule
Method 
& Sampling config.
& Runtime (s/test case)
& Speedup \\
\midrule
DiffusionPDE  & --             & 112.5 & \(1.73\times\) \\
FunDPS        & --             & 100.7 & \(1.55\times\) \\
DDIS          & \((100,5,20)\) & 153.9 & \(2.37\times\) \\
FLAPS-FNO      & \((20,10,40)\) & 65.1  & \(1.00\times\) \\
FLAPS-UNet     & \((20,10,40)\) & 64.9  & \(1.00\times\) \\
\bottomrule
\end{tabular}
\end{table}

\clearpage
\section{Ablation and scaling study on Low-Rank Covariance Preconditioning}
\label{app:lowrank-cov-ablation}

We study the effect of the low-rank covariance preconditioner used in the FLAPS Langevin correction step. In the masking regression experiment, the Langevin correction uses a preconditioned likelihood gradient of the form
\[
    \nabla_{u_1} \log p(y \mid {u_1})
    \quad \longmapsto \quad
    \widehat{\Sigma}_r \nabla_{u_1} \log p(y \mid u_1),
\]
where \(\widehat{\Sigma}_r\) is a rank-\(r\) covariance surrogate estimated from clean samples generated by the learned OFM prior. When \(r=0\), no covariance surrogate is used, and the preconditioner reduces to the identity matrix. This corresponds to the unpreconditioned Langevin baseline.

Table~\ref{tab:lowrank-cov-ablation} reports the posterior sample quality for different ranks of Matérn Gaussian process regression (resolution = 512). We compare FLAPS posterior samples against samples from the exact Gaussian process posterior using KL divergence, sliced Wasserstein distance (SWD), and maximum mean discrepancy (MMD). All metrics are computed using 128 posterior samples.

\begin{table}[h]
\centering
\caption{Ablation of the low-rank covariance preconditioner. Rank \(0\) corresponds to the identity preconditioner. Lower is better for all metrics.}
\label{tab:lowrank-cov-ablation}
\begin{tabular}{c c c c c}
\hline
Rank & KL & KL / dim & SWD & MMD \\
\hline
0  & \(9.41 \cdot 10^{4}\) & \(1.84\cdot10^{2}\) & \(6.63\cdot10^{-1}\) & \(5.08\cdot10^{-1}\) \\
2  & \(2.53\cdot10^{4}\) & \(4.93\cdot10^{1}\) & \(3.45\cdot10^{-1}\) & \(3.07\cdot10^{-1}\) \\
4  & \(5.28\cdot10^{3}\) & \(1.03\cdot10^{1}\) & \(6.26\cdot10^{-2}\) & \(4.20\cdot10^{-2}\) \\
8  & \(9.03\cdot10^{2}\) & \(1.76\cdot10^{0}\) & \(4.70\cdot10^{-2}\) & \(3.98\cdot10^{-2}\) \\
16 & \(3.74\cdot10^{2}\) & \(7.30\cdot10^{-1}\) & \(4.85\cdot10^{-2}\) & \(4.21\cdot10^{-2}\) \\
32 & \(3.34\cdot10^{2}\) & \(6.53\cdot10^{-1}\) & \(3.56\cdot10^{-2}\) & \(1.72\cdot10^{-2}\) \\
64 & \(3.43\cdot10^{2}\) & \(6.71\cdot10^{-1}\) & \(4.50\cdot10^{-2}\) & \(3.02\cdot10^{-2}\) \\
\hline
\end{tabular}
\end{table}

\begin{figure*}[ht]
    \vspace*{-0.3cm}
    \centering
    \includegraphics[width=0.9\textwidth]{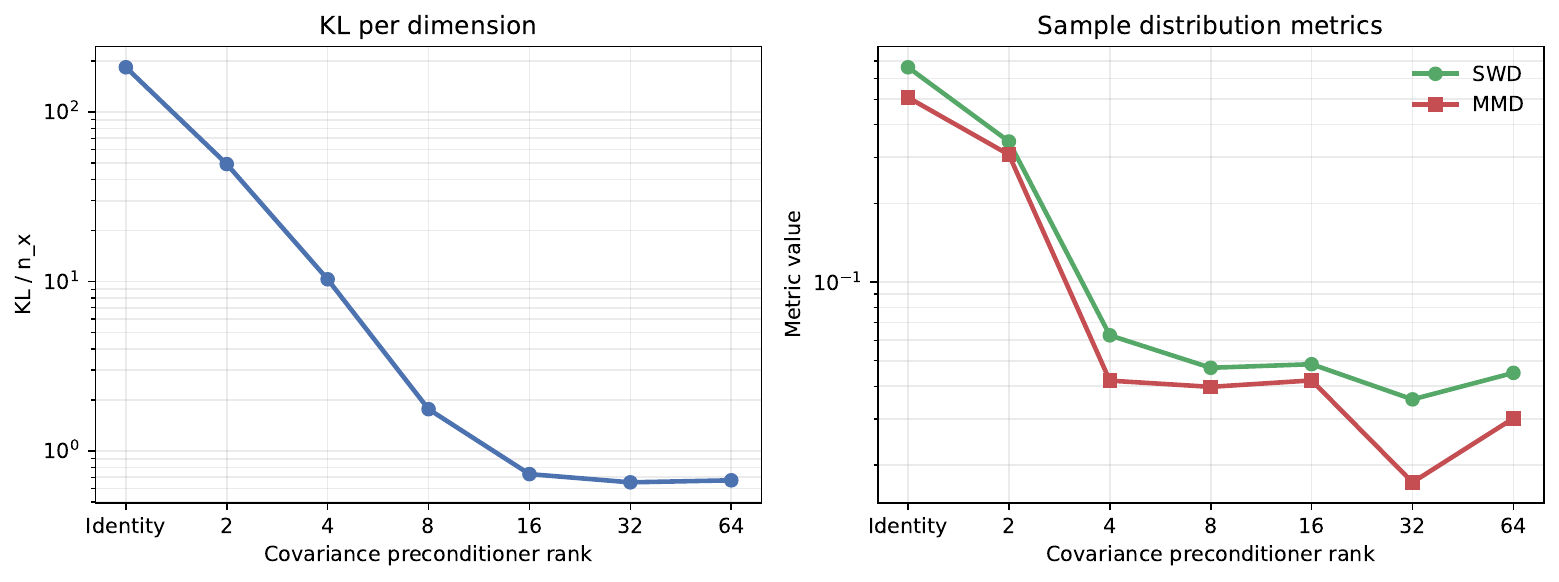}
    \caption{Ablation (rank=0) and scaling study of the low-rank covariance preconditioning}
    \label{fig:app_lowrank}
\end{figure*}

The results show that covariance preconditioning is crucial. The identity-preconditioned baseline (\(r=0\)) gives substantially worse posterior samples, with KL per dimension \(183.86\), SWD \(0.663\), and MMD \(0.508\). Introducing even a very low-rank covariance approximation improves all metrics significantly. For example, rank \(4\) reduces KL per dimension from \(183.86\) to \(10.30\), and rank \(8\) further reduces it to \(1.76\). Performance saturates after moderate rank. Rank \(32\) gives the best overall result in this experiment, achieving the lowest KL per dimension (\(0.653\)), SWD (\(0.0356\)), and MMD (\(0.0172\)). Increasing the rank to \(64\) does not further improve the result, suggesting that the dominant posterior geometry is already captured by a moderate-rank covariance surrogate. The runtime is nearly unchanged across ranks, around $40 \sim 41$ seconds in this experiment. Thus, the low-rank covariance preconditioner provides a large improvement in posterior quality with negligible additional sampling cost once the covariance surrogate has been estimated.

\subsection{Effect of the number of observations on the preferred rank}
\label{app:lowrank-rank-vs-nobs}

We next investigate whether the preferred covariance rank changes as the
number of observations increases. We consider
\(n_{\mathrm{obs}}\in\{8,16,24,32\}\), while keeping the inference
resolution fixed at \(512\). All other sampling and evaluation
hyperparameters are unchanged. Similarly to previous setting, each metric is computed
from 128 posterior samples, and rank \(0\) denotes the identity
preconditioner.

Tables~\ref{tab:lowrank-nobs16} and~\ref{tab:lowrank-nobs32} report the
complete results for \(n_{\mathrm{obs}}=16\) and \(n_{\mathrm{obs}}=32\),
respectively. The best value of each metric is highlighted in bold.
\begin{table}[ht]
\centering
\caption{Low-rank covariance ablation with \(n_{\mathrm{obs}}=16\) and
\(n_x=512\). Rank \(0\) is the identity preconditioner. Lower is better.}
\label{tab:lowrank-nobs16}
\begin{tabular}{c c c c c}
\hline
Rank & KL & KL / dim & SWD & MMD \\
\hline
0  & \(2.96\cdot10^{5}\) & \(5.78\cdot10^{2}\)
   & \(6.13\cdot10^{-1}\) & \(4.70\cdot10^{-1}\) \\
4  & \(1.81\cdot10^{4}\) & \(3.54\cdot10^{1}\)
   & \(1.50\cdot10^{-1}\) & \(1.45\cdot10^{-1}\) \\
8  & \(2.91\cdot10^{3}\) & \(5.69\cdot10^{0}\)
   & \(4.98\cdot10^{-2}\) & \(6.31\cdot10^{-2}\) \\
16 & \(1.16\cdot10^{3}\) & \(2.26\cdot10^{0}\)
   & \(4.81\cdot10^{-2}\) & \(6.15\cdot10^{-2}\) \\
24 & \(1.05\cdot10^{3}\) & \(2.05\cdot10^{0}\)
   & \(4.95\cdot10^{-2}\) & \(6.35\cdot10^{-2}\) \\
32 & \(1.04\cdot10^{3}\) & \(2.03\cdot10^{0}\)
   & \(4.91\cdot10^{-2}\) & \(6.24\cdot10^{-2}\) \\
48 & \(1.03\cdot10^{3}\) & \(2.02\cdot10^{0}\)
   & \(\mathbf{4.77\cdot10^{-2}}\)
   & \(\mathbf{6.11\cdot10^{-2}}\) \\
64 & \(\mathbf{1.03\cdot10^{3}}\)
   & \(\mathbf{2.01\cdot10^{0}}\)
   & \(5.06\cdot10^{-2}\) & \(6.54\cdot10^{-2}\) \\
\hline
\end{tabular}
\end{table}

\begin{table}[ht]
\centering
\caption{Low-rank covariance ablation with \(n_{\mathrm{obs}}=32\) and
\(n_x=512\). Rank \(0\) is the identity preconditioner. Lower is better.}
\label{tab:lowrank-nobs32}
\begin{tabular}{c c c c c}
\hline
Rank & KL & KL / dim & SWD & MMD \\
\hline
0  & \(3.74\cdot10^{5}\) & \(7.31\cdot10^{2}\)
   & \(3.97\cdot10^{-1}\) & \(2.88\cdot10^{-1}\) \\
4  & \(2.46\cdot10^{4}\) & \(4.80\cdot10^{1}\)
   & \(1.80\cdot10^{-1}\) & \(1.66\cdot10^{-1}\) \\
8  & \(3.96\cdot10^{3}\) & \(7.74\cdot10^{0}\)
   & \(5.91\cdot10^{-2}\) & \(7.17\cdot10^{-2}\) \\
16 & \(1.56\cdot10^{3}\) & \(3.05\cdot10^{0}\)
   & \(4.88\cdot10^{-2}\) & \(5.80\cdot10^{-2}\) \\
24 & \(1.41\cdot10^{3}\) & \(2.75\cdot10^{0}\)
   & \(\mathbf{4.46\cdot10^{-2}}\)
   & \(\mathbf{5.22\cdot10^{-2}}\) \\
32 & \(1.39\cdot10^{3}\) & \(2.72\cdot10^{0}\)
   & \(5.70\cdot10^{-2}\) & \(7.06\cdot10^{-2}\) \\
48 & \(\mathbf{1.39\cdot10^{3}}\)
   & \(\mathbf{2.71\cdot10^{0}}\)
   & \(5.10\cdot10^{-2}\) & \(6.33\cdot10^{-2}\) \\
64 & \(1.42\cdot10^{3}\) & \(2.76\cdot10^{0}\)
   & \(6.21\cdot10^{-2}\) & \(6.38\cdot10^{-2}\) \\
\hline
\end{tabular}
\end{table}

\begin{figure*}[ht]
    \centering
    \includegraphics[width=0.82\textwidth]
    {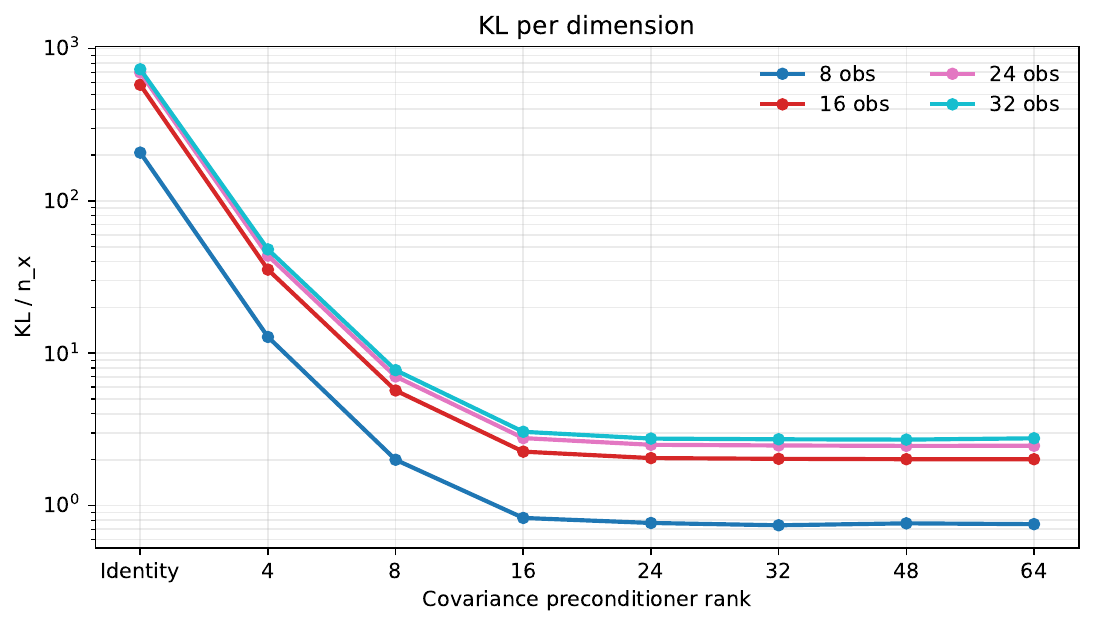}
    \caption{KL divergence per dimension as a function of covariance
    rank for different numbers of observations.}
    \label{fig:lowrank-nobs-kl}
\end{figure*}

\begin{figure*}[ht]
    \centering
    \includegraphics[width=0.82\textwidth]
    {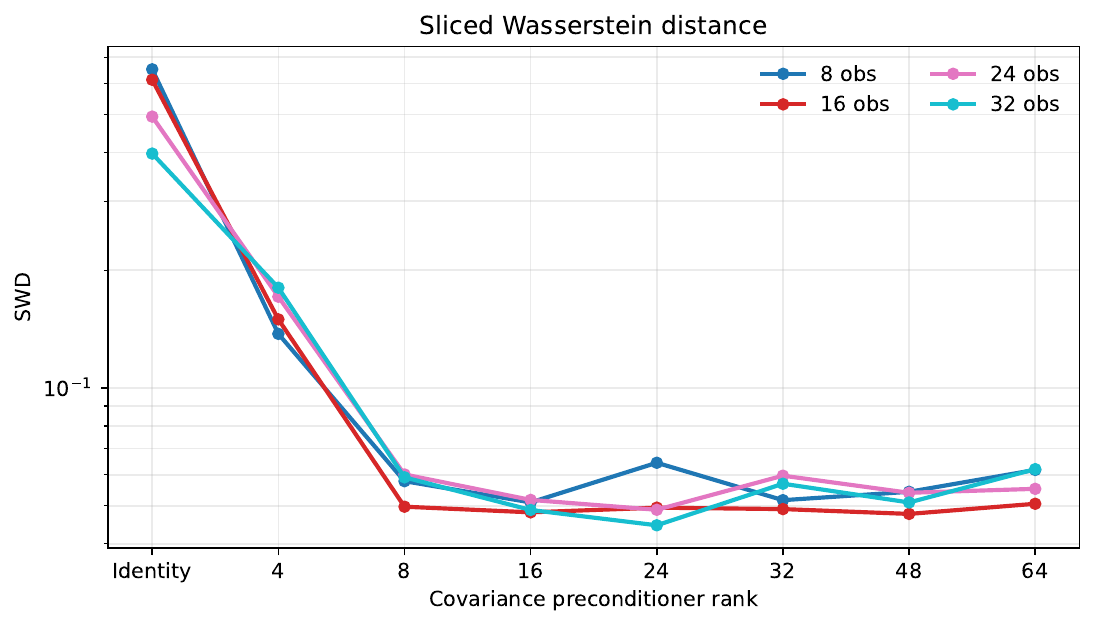}
    \caption{Sliced Wasserstein distance as a function of covariance
    rank for different numbers of observations.}
    \label{fig:lowrank-nobs-swd}
\end{figure*}

\begin{figure*}[ht]
    \centering
    \includegraphics[width=0.82\textwidth]
    {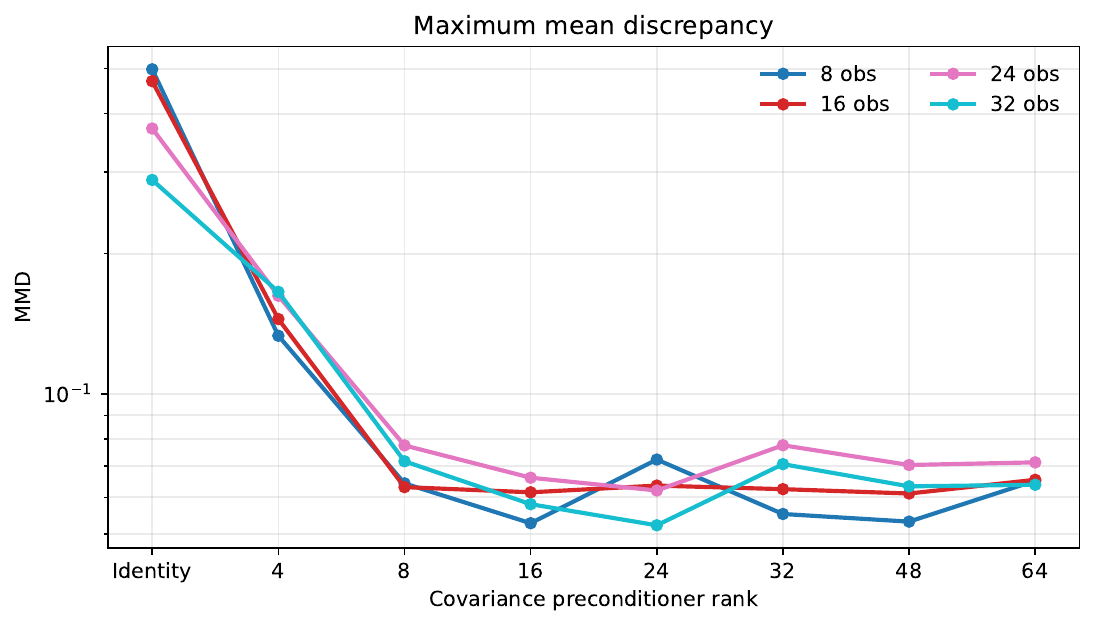}
    \caption{Maximum mean discrepancy as a function of covariance rank
    for different numbers of observations.}
    \label{fig:lowrank-nobs-mmd}
\end{figure*}

Figures~\ref{fig:lowrank-nobs-kl}--\ref{fig:lowrank-nobs-mmd} show that
covariance preconditioning remains important for every observation
count. Rank \(4\) already provides a substantial improvement over the
identity preconditioner, while the largest gain generally occurs when
the rank is increased from \(4\) to \(8\). Increasing the rank from
\(8\) to \(16\) produces a further clear improvement. Beyond rank
\(16\), the KL curves flatten, indicating that most of the useful prior
covariance structure has already been captured.

Although the rank attaining the numerical minimum varies across metrics
and observation counts, the results exhibit a clear and consistent
plateau once the rank reaches approximately \(24\)--\(32\). For KL per
dimension, the minimizing ranks are \(32\), \(64\), \(48\), and \(48\)
for \(n_{\mathrm{obs}}=8,16,24,\) and \(32\), respectively. However,
the differences among ranks \(32\), \(48\), and \(64\) are small compared
with the improvement obtained when moving from the identity or very
low-rank preconditioners into this moderate-rank regime. Thus, the
individual minimizing ranks should not be interpreted as evidence that
a separately tuned rank is required for each observation count.

A similar saturation pattern is visible in SWD and MMD. Their exact
minima fluctuate between ranks \(16\), \(24\), and \(48\), but ranks
above \(24\) generally remain in the same performance regime. These
small non-monotone variations are expected from finite-sample
estimation of both the covariance surrogate and the posterior metrics.
More importantly, increasing the rank beyond \(32\) does not produce a
consistent improvement across KL, SWD, and MMD. A larger rank may
slightly improve one metric for a particular observation count while
leaving the other metrics unchanged or making them slightly worse.

Rank \(32\) therefore provides a reasonable common choice for all tested
observation counts. It attains the lowest KL per dimension for
\(n_{\mathrm{obs}}=8\), lies very close to the minimum KL value for
\(n_{\mathrm{obs}}=16,24,\) and \(32\), and remains in the
plateau region for SWD and MMD. In contrast, selecting ranks \(48\) or
\(64\) increases the size of the covariance surrogate without yielding
a systematic improvement in posterior sample quality.

Overall, the study indicates that the important effect is reaching a
sufficiently expressive moderate-rank preconditioner, rather than
precisely tuning the rank for each number of observations. Performance
improves sharply between ranks \(0\) and \(16\), begins to saturate near
rank \(24\), and remains broadly stable for ranks \(32\) and above.
Accordingly, we use \(r=32\) as the default low-rank covariance setting:
it captures the dominant covariance structure, performs consistently
across all tested observation counts, and avoids the additional
complexity of higher-rank surrogates.

\section{Ablation Study of Practical Sampling Hyperparameters}
\label{app:sampling-ablation}

\subsection{Functional regression with a Matérn GP prior}

This study examines whether FLAPS is sensitive to the practical sampling
parameters in Algorithm~1 and whether the resulting posterior approximation
remains stable when the exact posterior is available. We use one-dimensional
Matérn GP regression because its exact posterior is analytically tractable,
allowing direct comparison between FLAPS samples and exact posterior samples
using sliced Wasserstein distance (SWD) and maximum mean discrepancy (MMD).

All configurations are evaluated on the same 20 test cases at resolution
\(n_x=128\), with 128 posterior samples generated per case. We perform
one-at-a-time ablations around the fixed configuration
\[
    (N_{\mathrm{anneal}},N_{\mathrm{ODE}},N_{\mathrm{Langevin}})
    =(40,20,50),
    \qquad
    (\lambda_{\min},\lambda_{\mathrm{scale}})=(0.05,1),
    \qquad r=32.
\]
This configuration is used uniformly across all functional-regression
experiments, including the non-Gaussian datasets. It was not selected
specifically for the Matérn GP benchmark or tuned to minimize the metrics
reported below. In the tables, \(^{\dagger}\) denotes this universal
configuration.

The reported SWD and MMD values quantify the end-to-end discrepancy between
FLAPS samples and the exact GP posterior. This discrepancy contains the
combined effects of the learned prior, the low-rank covariance approximation,
the Gaussian local correction, numerical ODE integration, finite-step
Langevin dynamics, and posterior-sampling error. The results therefore provide
a direct empirical measure of the overall posterior approximation error, but
do not isolate the Gaussian-local-conditional approximation alone.

\paragraph{Annealing schedule.}
Table~\ref{tab:matern-annealing-ablation} shows a moderate and non-monotone
dependence on the number of annealing levels. SWD ranges from \(0.0857\) to
\(0.104\), while MMD ranges from \(0.0936\) to \(0.115\). The lowest values
among the tested configurations occur at 10 annealing steps, whereas the
universal 40-step configuration remains in the same overall performance
regime. Importantly, the existence of a better-performing configuration shows
that the main functional-regression results were not obtained by selecting
the most favorable annealing schedule. Moreover, all tested schedules remain
numerically stable and produce meaningful posterior samples, indicating that
the method is not narrowly tuned to a single annealing schedule.

\paragraph{ODE discretization.}
FLAPS is largely insensitive to the number of ODE steps used for endpoint
transport. Across \(5\), \(10\), \(20\), and \(40\) steps, SWD varies only
from \(0.102\) to \(0.105\), while MMD varies from \(0.114\) to \(0.116\).
Thus, the endpoint transport is already sufficiently resolved using a coarse
ODE discretization, and increasing the number of solver steps does not
materially change the posterior approximation. This indicates that the
reported performance is not driven by careful tuning of the ODE solver.

\paragraph{Langevin correction.}
The number of Langevin steps has a more visible effect. The posterior metrics
improve from 10 to 20 steps and then gradually deteriorate as the count is
increased to 80. Nevertheless, all tested configurations produce finite and
meaningful posterior samples. The universal 50-step setting is not the best
configuration found in the sweep: using 20 steps reduces SWD from \(0.104\)
to \(0.0900\) and MMD from \(0.115\) to \(0.0977\). This again shows that the
main results were obtained using a shared configuration rather than a
task-specific optimum. The non-monotone trend may reflect a trade-off between
insufficient likelihood correction at small step counts and accumulated
discretization and stochastic-sampling effects at larger step counts.

\paragraph{Anchor schedule.}
The local correction uses
\[
    \lambda(t)
    =
    \max\left\{
        \lambda_{\min},
        \lambda_{\mathrm{scale}}(1-t)
    \right\}.
\]
The \(\lambda_{\mathrm{scale}}\) ablation identifies a broad useful regime
rather than a single narrowly stable value. The lowest discrepancy among the
tested settings occurs at \(\lambda_{\mathrm{scale}}=0.5\), while the
universal value \(\lambda_{\mathrm{scale}}=1\) remains competitive. A much
smaller value, \(0.2\), performs worse because the corresponding
\(1/\lambda(t)^2\) restoring force constrains the samples too strongly around
the endpoint-flow prediction. Increasing \(\lambda_{\mathrm{scale}}\) beyond
\(0.5\) weakens this anchoring and gradually increases the posterior
discrepancy. These results show that the correction requires a reasonable
anchor scale, but is not sensitive to a narrowly specified value within the
moderate regime.

The results are even less sensitive to \(\lambda_{\min}\). Values between
\(0.04\) and \(0.08\) produce nearly indistinguishable SWD and MMD, and the
universal value \(0.05\) lies inside this plateau. The gradual
deterioration at larger values confirms that \(\lambda_{\min}\) primarily
controls the terminal portion of the correction.

\paragraph{Implications for robustness.}
Taken together, these experiments show that FLAPS is not heavily tuned to a
single sampling configuration. The ODE discretization and
\(\lambda_{\min}\) exhibit broad performance plateaus, while the annealing,
Langevin, and anchor-scale ablations remain stable over substantial parameter
ranges. Several configurations improve SWD and MMD over the universal
setting, further demonstrating that the main results were not obtained by
selecting benchmark-specific optima.

We therefore retain the universal functional-regression configuration for
all functional-regression experiments, including both Gaussian and
non-Gaussian datasets, rather than replacing it with a collection of
one-at-a-time optima. The purpose of this ablation is to characterize the
robustness of practical FLAPS and expose the available performance margin, not
to tune a separate configuration for each benchmark.

\subsection{Ablation Study for the Poisson Inverse Problem}
\label{app:poisson-hyperparameter-ablation}

We further assess the sensitivity of FLAPS to the practical sampling
parameters in Algorithm~1 on a nonlinear PDE inverse problem. We use the
first 10 Poisson test cases and generate 32 posterior samples per case. CRPS
and SSR are computed from the complete posterior ensemble, while PSNR, SSIM,
and Relative \(L_2\) use the same fixed posterior draw for each test case
across all configurations. Lower CRPS and Relative \(L_2\) are better, higher
PSNR and SSIM are better, and SSR is best when closest to one.

Except for the parameter being varied, we use
\[
 (N_{\mathrm{anneal}},N_{\mathrm{ODE}},N_{\mathrm{Langevin}})
 =(20,10,40),\qquad
 (\lambda_{\min},\lambda_{\mathrm{scale}})=(0.05,1),\qquad r=32.
\]
This configuration is shared across all PDE inverse experiments and was not
tuned specifically for Poisson. Bold values indicate the best result according
to the full-precision values, while \(^{\dagger}\) marks the universal
configuration.

\paragraph{Annealing steps.}
Increasing the number of annealing steps generally improves reconstruction
quality. Relative to 20 steps, 80 steps reduce CRPS by \(9.9\%\) and Relative
\(L_2\) by \(4.5\%\). Posterior calibration remains stable across the sweep,
with SSR remaining within approximately \(2\%\) of one. The universal
20-step setting gives the closest SSR to one, while additional annealing steps
provide modest improvements in the reconstruction metrics at higher
computational cost.

\paragraph{ODE steps.}
CRPS changes by less than \(0.3\%\) across the tested range, showing that the
endpoint transport is already well resolved with a coarse ODE discretization.
Five ODE steps give the strongest reconstruction metrics, while the universal
10-step setting gives the best calibration. Increasing the number of ODE
steps beyond 10 provides no meaningful benefit.

\paragraph{Anchor-scale parameter.}
All metrics vary only slightly with \(\lambda_{\mathrm{scale}}\). Although the
full-precision values improve monotonically up to
\(\lambda_{\mathrm{scale}}=2\), the differences between 1 and 2 are only
approximately \(0.3\%\) for both CRPS and Relative \(L_2\), while SSR remains
close to one throughout the sweep. The universal value
\(\lambda_{\mathrm{scale}}=1\) therefore lies within a broad stable regime and
does not require problem-specific tuning.

\paragraph{Langevin steps.}
The number of Langevin steps has the clearest effect. Ten steps provide
insufficient likelihood correction and produce an overdispersed posterior.
Increasing the number of steps improves all reconstruction metrics and moves
SSR toward one. The universal 40-step setting gives the closest SSR to one,
while 80 steps reduce CRPS and Relative \(L_2\) by \(6.0\%\) and \(5.6\%\),
respectively, at twice the likelihood-gradient cost.

\paragraph{Overall robustness.}
The Poisson ablations further show that FLAPS is not narrowly tuned to a
problem-specific sampling configuration. ODE discretization and
\(\lambda_{\mathrm{scale}}\) exhibit broad performance plateaus, while
posterior calibration remains stable across the tested annealing schedules.
The clearest sensitivity is to using too few Langevin updates, which leads to
incomplete likelihood correction.

Although more expensive configurations improve several Poisson metrics, we
retain
\[
 (N_{\mathrm{anneal}},N_{\mathrm{ODE}},N_{\mathrm{Langevin}})
 =(20,10,40),\qquad
 (\lambda_{\min},\lambda_{\mathrm{scale}})=(0.05,1),\qquad r=32
\]
for all PDE inverse experiments. This universal configuration provides SSR
close to one and competitive reconstruction quality without task-specific
tuning. The existence of more expensive settings that further improve the
Poisson results also confirms that the main results were not obtained by
selecting the best configuration separately for each benchmark.
\begin{table}[ht]
\centering
\caption{Ablation of the number of annealing steps. Lower is better.}
\label{tab:matern-annealing-ablation}
\begin{tabular}{c c c c c}
\hline
Metric & 5 & 10 & 20 & 40 \\
\hline
SWD
& \(8.93\cdot10^{-2}\)
& \(\mathbf{8.57\cdot10^{-2}}\)
& \(8.93\cdot10^{-2}\)
& \(1.04\cdot10^{-1}\) \\
MMD
& \(9.80\cdot10^{-2}\)
& \(\mathbf{9.36\cdot10^{-2}}\)
& \(9.61\cdot10^{-2}\)
& \(1.15\cdot10^{-1}\) \\
\hline
\end{tabular}
\end{table}

\begin{table}[ht]
\centering
\caption{Ablation of the number of ODE steps. Lower is better.}
\label{tab:matern-ode-ablation}
\begin{tabular}{c c c c c}
\hline
Metric & 5 & 10 & 20 & 40 \\
\hline
SWD
& \(\mathbf{1.02\cdot10^{-1}}\)
& \(1.03\cdot10^{-1}\)
& \(1.04\cdot10^{-1}\)
& \(1.05\cdot10^{-1}\) \\
MMD
& \(\mathbf{1.14\cdot10^{-1}}\)
& \(1.15\cdot10^{-1}\)
& \(1.15\cdot10^{-1}\)
& \(1.16\cdot10^{-1}\) \\
\hline
\end{tabular}
\end{table}

\begin{table}[ht]
\centering
\caption{Ablation of the number of Langevin steps. Lower is better.}
\label{tab:matern-langevin-ablation}
\begin{tabular}{c c c c c c}
\hline
Metric & 10 & 20 & 40 & 50 & 80 \\
\hline
SWD
& \(1.11\cdot10^{-1}\)
& \(\mathbf{9.00\cdot10^{-2}}\)
& \(1.00\cdot10^{-1}\)
& \(1.04\cdot10^{-1}\)
& \(1.15\cdot10^{-1}\) \\
MMD
& \(1.25\cdot10^{-1}\)
& \(\mathbf{9.77\cdot10^{-2}}\)
& \(1.10\cdot10^{-1}\)
& \(1.15\cdot10^{-1}\)
& \(1.28\cdot10^{-1}\) \\
\hline
\end{tabular}
\end{table}

\begin{table}[ht]
\centering
\caption{Ablation of \(\lambda_{\mathrm{scale}}\) (default=1). Lower is better.}
\label{tab:matern-lambda-scale-ablation}
\begin{tabular}{c c c c c c}
\hline
Metric & 0.2 & 0.5 & 1 & 1.5 & 2 \\
\hline
SWD
& \(1.52\cdot10^{-1}\)
& \(\mathbf{9.38\cdot10^{-2}}\)
& \(1.04\cdot10^{-1}\)
& \(1.14\cdot10^{-1}\)
& \(1.20\cdot10^{-1}\) \\
MMD
& \(1.72\cdot10^{-1}\)
& \(\mathbf{1.04\cdot10^{-1}}\)
& \(1.15\cdot10^{-1}\)
& \(1.23\cdot10^{-1}\)
& \(1.27\cdot10^{-1}\) \\
\hline
\end{tabular}
\end{table}

\begin{table*}[ht]
\centering
\caption{Ablation of \(\lambda_{\min}\) (default=0.05). Lower is better.}
\label{tab:matern-lambda-min-ablation}
\resizebox{\textwidth}{!}{
\begin{tabular}{c c c c c c c c c}
\hline
Metric & 0.025 & 0.04 & 0.05 & 0.08 & 0.1 & 0.15 & 0.2 \\
\hline
SWD

& \(1.12\cdot10^{-1}\)
& \(1.04\cdot10^{-1}\)
& \(\mathbf{1.04\cdot10^{-1}}\)
& \(1.05\cdot10^{-1}\)
& \(1.06\cdot10^{-1}\)
& \(1.08\cdot10^{-1}\)
& \(1.13\cdot10^{-1}\) \\
MMD

& \(1.16\cdot10^{-1}\)
& \(\mathbf{1.15\cdot10^{-1}}\)
& \(1.15\cdot10^{-1}\)
& \(1.16\cdot10^{-1}\)
& \(1.17\cdot10^{-1}\)
& \(1.19\cdot10^{-1}\)
& \(1.22\cdot10^{-1}\) \\
\hline
\end{tabular}}
\end{table*}

\begin{table}[ht]
\centering
\caption{Sensitivity to the number of annealing steps (Inverse PDE).}
\label{tab:poisson-annealing-ablation}
\begin{tabular}{c c c c c}
\hline
Metric & \(20^{\dagger}\) & 40 & 60 & 80 \\
\hline
CRPS
& \(8.70\cdot10^{-2}\) & \(8.26\cdot10^{-2}\)
& \(8.07\cdot10^{-2}\) & \(\mathbf{7.84\cdot10^{-2}}\) \\
SSR
& \(\mathbf{1.01\cdot10^{0}}\) & \(1.02\cdot10^{0}\)
& \(9.90\cdot10^{-1}\) & \(1.01\cdot10^{0}\) \\
PSNR
& \(2.72\cdot10^{1}\) & \(2.70\cdot10^{1}\)
& \(2.70\cdot10^{1}\) & \(\mathbf{2.75\cdot10^{1}}\) \\
SSIM
& \(5.73\cdot10^{-1}\) & \(5.64\cdot10^{-1}\)
& \(5.64\cdot10^{-1}\) & \(\mathbf{6.03\cdot10^{-1}}\) \\
Relative \(L_2\)
& \(5.63\cdot10^{-1}\) & \(5.72\cdot10^{-1}\)
& \(5.66\cdot10^{-1}\) & \(\mathbf{5.37\cdot10^{-1}}\) \\
\hline
\end{tabular}
\end{table}

\begin{table}[ht]
\centering
\caption{Sensitivity to the number of ODE steps (Inverse PDE). }
\label{tab:poisson-ode-ablation}
\begin{tabular}{c c c c c}
\hline
Metric & 5 & \(10^{\dagger}\) & 20 & 40 \\
\hline
CRPS
& \(8.72\cdot10^{-2}\) & \(8.70\cdot10^{-2}\)
& \(8.70\cdot10^{-2}\) & \(\mathbf{8.70\cdot10^{-2}}\) \\
SSR
& \(9.57\cdot10^{-1}\) & \(\mathbf{1.01\cdot10^{0}}\)
& \(1.03\cdot10^{0}\) & \(1.05\cdot10^{0}\) \\
PSNR
& \(\mathbf{2.74\cdot10^{1}}\) & \(2.72\cdot10^{1}\)
& \(2.71\cdot10^{1}\) & \(2.70\cdot10^{1}\) \\
SSIM
& \(\mathbf{5.81\cdot10^{-1}}\) & \(5.73\cdot10^{-1}\)
& \(5.69\cdot10^{-1}\) & \(5.67\cdot10^{-1}\) \\
Relative \(L_2\)
& \(\mathbf{5.51\cdot10^{-1}}\) & \(5.63\cdot10^{-1}\)
& \(5.70\cdot10^{-1}\) & \(5.73\cdot10^{-1}\) \\
\hline
\end{tabular}
\end{table}

\begin{table}[ht]
\centering
\caption{Sensitivity to \(\lambda_{\mathrm{scale}}\), with
\(\lambda_{\min}=0.05\) (Inverse PDE).}
\label{tab:poisson-lambda-scale-ablation}
\begin{tabular}{c c c c c c}
\hline
Metric & 0.2 & 0.5 & \(1^{\dagger}\) & 1.5 & 2 \\
\hline
CRPS
& \(8.83\cdot10^{-2}\) & \(8.74\cdot10^{-2}\)
& \(8.70\cdot10^{-2}\) & \(8.68\cdot10^{-2}\)
& \(\mathbf{8.67\cdot10^{-2}}\) \\
SSR
& \(1.01\cdot10^{0}\) & \(1.01\cdot10^{0}\)
& \(1.01\cdot10^{0}\) & \(1.01\cdot10^{0}\)
& \(\mathbf{1.01\cdot10^{0}}\) \\
PSNR
& \(2.70\cdot10^{1}\) & \(2.71\cdot10^{1}\)
& \(2.72\cdot10^{1}\) & \(2.72\cdot10^{1}\)
& \(\mathbf{2.72\cdot10^{1}}\) \\
SSIM
& \(5.68\cdot10^{-1}\) & \(5.72\cdot10^{-1}\)
& \(5.73\cdot10^{-1}\) & \(5.74\cdot10^{-1}\)
& \(\mathbf{5.74\cdot10^{-1}}\) \\
Relative \(L_2\)
& \(5.71\cdot10^{-1}\) & \(5.65\cdot10^{-1}\)
& \(5.63\cdot10^{-1}\) & \(5.62\cdot10^{-1}\)
& \(\mathbf{5.61\cdot10^{-1}}\) \\
\hline
\end{tabular}
\end{table}

\begin{table}[ht]
\centering
\caption{Sensitivity to the number of Langevin steps (Inverse PDE).}
\label{tab:poisson-langevin-ablation}
\begin{tabular}{c c c c c}
\hline
Metric & 10 & 20 & \(40^{\dagger}\) & 80 \\
\hline
CRPS
& \(9.92\cdot10^{-2}\) & \(9.34\cdot10^{-2}\)
& \(8.70\cdot10^{-2}\) & \(\mathbf{8.18\cdot10^{-2}}\) \\
SSR
& \(1.09\cdot10^{0}\) & \(1.02\cdot10^{0}\)
& \(\mathbf{1.01\cdot10^{0}}\) & \(9.93\cdot10^{-1}\) \\
PSNR
& \(2.57\cdot10^{1}\) & \(2.68\cdot10^{1}\)
& \(2.72\cdot10^{1}\) & \(\mathbf{2.77\cdot10^{1}}\) \\
SSIM
& \(5.13\cdot10^{-1}\) & \(5.68\cdot10^{-1}\)
& \(5.73\cdot10^{-1}\) & \(\mathbf{5.96\cdot10^{-1}}\) \\
Relative \(L_2\)
& \(6.69\cdot10^{-1}\) & \(5.81\cdot10^{-1}\)
& \(5.63\cdot10^{-1}\) & \(\mathbf{5.31\cdot10^{-1}}\) \\
\hline
\end{tabular}
\end{table}

\clearpage
\newpage
\section{Comparison with Existing Methods}
\label{app:comparison_existing_methods}

\paragraph{Comparison with OFM posterior sampling and Neural Processes.}
FLAPS differs from neural-process conditional models and direct OFM posterior sampling. Neural Processes provide efficient amortized prediction, but do not explicitly separate reusable prior learning from test-time Bayesian posterior inference. OFM learns a valid function-space flow-matching prior with finite-dimensional marginals on arbitrary query sets, but direct OFM posterior sampling is less flexible for complex observation operators. FLAPS bridges this gap by converting a pretrained OFM prior into a likelihood-guided posterior sampler that avoids explicit prior-density evaluation and supports both functional regression and PDE inverse problems. Table~\ref{tab:faps_comparison_compact} summarizes the main distinctions.

\textbf{Computational time comparison with OFM posterior sampling.}
We benchmark computational efficiency on Matérn GP regression at resolution \(1024\), drawing \(512\) posterior samples with batch size \(32\). Both FLAPS and OFM posterior sampling use \(20\) ODE steps, and the Hutchinson batch size for OFM posterior sampling is set to \(32\). The model weights occupy \(0.62\) GB of GPU memory. FLAPS reaches a peak memory usage of \(0.97\) GB, corresponding to only \(0.35\) GB of additional runtime memory, and takes \(54.28\) seconds. 

Direct OFM posterior sampling is computationally expensive because evaluating the flow-prior likelihood requires integrating the full probability-flow trajectory and estimating the divergence term, which involves gradients of the velocity model at each ODE step. Consequently, OFM posterior sampling reaches \(17.9\) GB peak memory and takes \(1537.6\) seconds, about \(50\times\) higher runtime memory and \(29\times\) longer runtime than FLAPS. With the adaptive \texttt{dopri5} solver, its peak memory further increases to \(34.1\) GB and runtime to \(3924.64\) seconds, about \(98\times\) higher runtime memory and \(72\times\) longer runtime. These results demonstrate the computational advantage of FLAPS, which avoids explicit prior-density evaluation and expensive likelihood estimation.

\begin{table}[h]
\footnotesize
\centering
\renewcommand{\arraystretch}{1.15}
\caption{Computational efficiency comparison on Matérn GP regression at resolution \(1024\). 
We draw \(512\) posterior samples with batch size \(32\). The model weights occupy \(0.62\) GB of GPU memory; runtime memory denotes the additional memory beyond model weights.}
\label{tab:computational-efficiency}
\begin{tabular}{c c c c c c c}
\hline
Method & ODE solver & ODE steps & Peak memory & Runtime memory & Runtime & Memory / Time \\
\hline
FLAPS 
& Euler 
& \(20\) 
& \(0.97\) GB 
& \(0.35\) GB
& \(54.28\) s 
& \(1.0\times\) / \(1.0\times\) \\

OFM post. samp.
& Euler 
& \(20\) 
& \(17.9\) GB 
& \(17.3\) GB
& \(1537.6\) s 
& \(\sim 50\times\) / \(\sim 29\times\) \\

OFM post. samp. 
& \texttt{dopri5} 
& adaptive 
& \(34.1\) GB 
& \(33.5\) GB
& \(3924.64\) s 
& \(\sim 98\times\) / \(\sim 72\times\) \\
\hline
\end{tabular}
\label{fig:app_comp_cost}
\end{table}

\begin{table}[t]
\centering
\footnotesize
\renewcommand{\arraystretch}{1.22}
\caption{High-level comparison of FLAPS, OFM posterior sampling, and Neural Processes.}
\begin{tabularx}{0.98\linewidth}{lLLL}
\toprule
\textbf{Aspect} 
& \textbf{FLAPS (Ours)} 
& \textbf{OFM posterior sampling} 
& \textbf{Neural Processes} \\
\midrule

Core idea &
Likelihood-guided posterior sampling with pretrained function-space flow priors. &
Posterior inference from OFM-induced finite-dimensional marginals. &
Amortized conditional prediction from context to targets. \\

Prior / posterior separation &
Explicit reusable prior plus test-time posterior sampler. &
Explicit reusable prior; posterior sampling tied to OFM density/conditioning. &
No explicit reusable prior-posterior decomposition. \\

Prior density required &
No explicit prior-density evaluation. &
Typically uses tractable OFM marginal likelihoods. &
Not applicable. \\

Observation operator &
General \(y=\mathcal A(u_1)+\epsilon\). &
Direct function observations. &
Mostly direct context-target observations. \\

PDE inverse problems &
Naturally supported via \(P_\Omega\mathcal G_\phi\). &
Not naturally supported. &
Not naturally supported. \\

Posterior correction &
Flow transport + Langevin correction + re-bridging. &
Direct OFM-based posterior sampling. &
Learned conditional decoder. \\

Correlation-aware updates &
Low-rank covariance-preconditioned likelihood correction. &
No explicit correction preconditioner. &
Implicit through learned representation. \\

\bottomrule
\end{tabularx}
\label{tab:faps_comparison_compact}
\end{table}

\paragraph{Comparison with diffusion-based posterior sampling.}
\label{app:diffusion_comparison}

We compare FLAPS with representative diffusion-based posterior samplers for PDE inverse problems in Table~\ref{tab:diffusion_posterior_sampling_comparison}. Existing methods mainly differ in how the generative prior is trained and how the forward PDE model is incorporated during posterior inference. DiffusionPDE and FunDPS learn diffusion priors from PDE-generated data, often through joint input--solution representations or PDE-state distributions. This can be expensive when high-fidelity paired simulations are costly, and the resulting priors are typically specialized to a particular PDE family or discretization. According to~\citep{lin_decoupled_2026}, these guidance-based methods are plagued by a Jensen gap and over-smoothing, resulting in reconstructions that miss fine-grained physical details. DDIS mitigates this issue by decoupling the learned coefficient prior from a neural-operator forward surrogate. FLAPS follows the same decoupled principle, but replaces the diffusion prior with a function-space flow-matching prior whose finite-dimensional marginals are consistent across query sets. This allows FLAPS to perform posterior inference on variable discretizations and enables zero-shot PDE inverse inference on finer meshes not seen during prior training.

A second important distinction lies in the posterior correction step. 
Although DDIS is also decoupled, its DAPS-style Langevin correction injects isotropic white noise in the discretized coefficient space, making the correction grid-coordinate-wise rather than function-space-aware. 
This is suboptimal for sparse PDE inverse problems, where posterior uncertainty is highly correlated and the forward PDE map (e.g. differentiable PDE solver) can be sensitive to unrealistic high-frequency  (from white noise). 
FLAPS instead uses covariance-preconditioned Langevin dynamics, injecting sample-like (\(\mathcal N(0,C_X)\)) smooth noise and preconditioning the likelihood gradient by \(C_X\). 
As a result, posterior correction follows sample-like function-space correlations and propagates sparse observation information more coherently across the unknown field.

We further emphasize that Relative \(L^2\) (the primary metric used in previous study~\citep{huang_diffusionpde_2024, yao_guided_2026, lin_decoupled_2026}) is only a reconstruction metric and, by itself, is not a reliable measure of posterior sampling quality in sparse and noisy inverse problems. It measures the distance between a single reconstruction and one held-out ground-truth field, but does not assess whether the generated samples faithfully represent the posterior distribution \(p(u_1\mid y)\). This distinction is crucial because sparse PDE inverse problems are generally ill-posed: many input fields may be consistent with the same sparse observations. Consequently, methods that collapse to a posterior mean or MAP-like estimate may achieve a lower Relative \(L^2\) while failing to capture posterior diversity and uncertainty. Similar posterior-collapse behavior has been observed in diffusion posterior sampling, where samplers can concentrate on restricted solution sets even in simple Gaussian settings~\citep{zhang_improving_2025,xu_rethinking_2025}. Relative \(L^2\) becomes more informative only when observations are sufficiently dense, so that posterior uncertainty is small and the conditional distribution is close to a Dirac measure.

Since DiffusionPDE and FunDPS are both PDE-focused diffusion baselines trained from PDE-generated data, we use FunDPS as the representative function-space diffusion baseline in the qualitative comparison, while retaining DiffusionPDE in the quantitative reconstruction tables.

\begin{table*}[t]
\centering
\scriptsize
\renewcommand{\arraystretch}{1.18}
\setlength{\tabcolsep}{5pt}
\caption{
Compact comparison between FLAPS and representative diffusion-based PDE posterior samplers.
DiffusionPDE is included in the quantitative comparisons but omitted here since it is conceptually close to PDE-specific diffusion baselines.
}
\label{tab:diffusion_posterior_sampling_comparison}
\begin{tabularx}{\textwidth}{
p{0.16\textwidth}
>{\raggedright\arraybackslash}X
>{\raggedright\arraybackslash}X
>{\raggedright\arraybackslash}X}
\toprule
\textbf{Aspect}
& \textbf{FLAPS (Ours)}
& \textbf{FunDPS}
& \textbf{DDIS} \\
\midrule

Problem scope
& Functional regression + PDE inverse problems
& PDE inverse problems
& PDE inverse problems \\

Prior learning
& Decoupled flow prior; input fields only
& Joint diffusion prior; paired input--solution fields
& Decoupled diffusion prior; input fields only \\

Backbone
& FNO or UNet
& FNO
& UNet \\

Discretization
& Function-space or fixed-grid, depending on backbone
& Multi-resolution function-space diffusion
& Fixed-grid diffusion prior \\

Posterior sampling
& Endpoint transport + covariance Langevin + re-bridging
& Likelihood-guided reverse diffusion
& Decoupled annealed diffusion sampling \\

Correction geometry
& Correlated covariance; isotropic at rank 0
& Diffusion-score guidance
& Isotropic Langevin correction \\

Forward model
& Plug-in differentiable observation operator or PDE surrogate
& PDE-conditioned guidance
& Plug-in differentiable PDE surrogate \\

Practical flexibility
& Finite- or function-space flow priors
& Function-space diffusion prior
& Finite-space (fixed-grid) diffusion prior \\

Performance \& cost
& Competitive or better with lower test-time cost
& Competitive
& Strong on some PDE tasks \\

\bottomrule
\end{tabularx}
\end{table*}

\section{Additional Results}
\label{app:additional_results}

In this section, we provide additional quantitative and qualitative results. 
Table~\ref{tab:FAPS_GP_Gibbs} reports functional regression performance on the nonstationary Gibbs-kernel GP benchmark. 
FLAPS substantially outperforms the baselines and remains stable as the query resolution increases. 
In contrast, NDP, FlowNP, and direct OFM posterior sampling exhibit degraded posterior-distribution accuracy at higher resolutions.

For PDE inverse problems, we provide additional posterior samples in 
Figs.~\ref{fig:app_darcy}, \ref{fig:app_helmholtz}, \ref{fig:app_poisson}, and~\ref{fig:app_ns}. 
These examples further illustrate that FLAPS produces coherent posterior input fields and corresponding posterior predictive solution fields across different PDE inverse settings.

\begin{table*}[h]
\centering
\caption{Comparison with baseline models on one-dimensional GP regression
with a nonstationary Gibbs kernel. Lower SWD and MMD indicate better
agreement with the reference posterior. Values are mean
\(\pm\) standard deviation. Best values are in bold.}
\label{tab:FAPS_GP_Gibbs}

\scriptsize
\setlength{\tabcolsep}{2.5pt}
\renewcommand{\arraystretch}{1.08}

\resizebox{\textwidth}{!}{%
\begin{tabular}{@{}lcccc@{}}
\toprule
& \multicolumn{2}{c}{Query size \(128\)}
& \multicolumn{2}{c}{Query size \(512\)} \\
\cmidrule(lr){2-3}
\cmidrule(lr){4-5}
Algorithm $\downarrow$ Metric $\rightarrow$ & SWD & MMD & SWD & MMD \\
\midrule

TNP
& \(8.94{\cdot}10^{-1}\pm3.6{\cdot}10^{-1}\)
& \(6.53{\cdot}10^{-1}\pm1.6{\cdot}10^{-1}\)
& \(8.50{\cdot}10^{-1}\pm3.9{\cdot}10^{-1}\)
& \(6.32{\cdot}10^{-1}\pm1.4{\cdot}10^{-1}\) \\

CNP
& \(3.00{\cdot}10^{-1}\pm9.0{\cdot}10^{-2}\)
& \(2.92{\cdot}10^{-1}\pm8.9{\cdot}10^{-2}\)
& \(2.78{\cdot}10^{-1}\pm8.4{\cdot}10^{-2}\)
& \(2.67{\cdot}10^{-1}\pm9.2{\cdot}10^{-2}\) \\

ANP
& \(2.70{\cdot}10^{-1}\pm6.7{\cdot}10^{-2}\)
& \(2.19{\cdot}10^{-1}\pm8.0{\cdot}10^{-2}\)
& \(2.78{\cdot}10^{-1}\pm7.2{\cdot}10^{-2}\)
& \(2.12{\cdot}10^{-1}\pm7.8{\cdot}10^{-2}\) \\

NDP
& \(2.62{\cdot}10^{-1}\pm1.2{\cdot}10^{-1}\)
& \(2.69{\cdot}10^{-1}\pm1.3{\cdot}10^{-1}\)
& \(4.39{\cdot}10^{-1}\pm1.6{\cdot}10^{-1}\)
& \(4.66{\cdot}10^{-1}\pm1.7{\cdot}10^{-1}\) \\

FlowNP
& \(3.47{\cdot}10^{-1}\pm2.0{\cdot}10^{-1}\)
& \(3.58{\cdot}10^{-1}\pm2.1{\cdot}10^{-1}\)
& \(3.05{\cdot}10^{-1}\pm8.5{\cdot}10^{-2}\)
& \(2.61{\cdot}10^{-1}\pm8.5{\cdot}10^{-2}\) \\

OFM
& \(2.87{\cdot}10^{-1}\pm1.4{\cdot}10^{-1}\)
& \(2.62{\cdot}10^{-1}\pm1.1{\cdot}10^{-1}\)
& \(3.69{\cdot}10^{-1}\pm1.6{\cdot}10^{-1}\)
& \(3.63{\cdot}10^{-1}\pm1.8{\cdot}10^{-1}\) \\

\rowcolor{yellow!25}
\textbf{FLAPS (Ours)}
& \(\mathbf{1.94{\cdot}10^{-1}}\pm9.9{\cdot}10^{-2}\)
& \(\mathbf{1.67{\cdot}10^{-1}}\pm8.7{\cdot}10^{-2}\)
& \(\mathbf{1.55{\cdot}10^{-1}}\pm7.4{\cdot}10^{-2}\)
& \(\mathbf{1.36{\cdot}10^{-1}}\pm7.4{\cdot}10^{-2}\) \\

\bottomrule
\end{tabular}%
}
\end{table*}

\section{Zero-shot Super-resolution for PDE Inverse Problems}
\label{app:zero_shot_super_resolution}

In this section, we evaluate whether FLAPS can perform PDE inverse inference on query resolutions higher than those used during training. 
For Darcy flow, we train the OFM prior and the FNO PDE surrogate at the base resolution, and then perform posterior sampling directly on a \(160\times160\) query grid without retraining either model. 
Sparse noisy observations are taken from the corresponding solution field, and FLAPS is used to infer the high-resolution input coefficient field. 

As shown in Fig.~\ref{fig:sup_darcy}, the posterior predictive solutions remain consistent with the observed solution values, while the inferred input-field samples preserve the sharp interface structure of the Darcy coefficient. 
The posterior uncertainty is elevated near interfaces and other ambiguous regions, where sparse solution observations do not fully determine the input field. 
These results demonstrate that FLAPS can combine function-space priors with likelihood-guided correction to enable zero-shot super-resolution PDE inverse inference.

\begin{figure*}[ht]
    \vspace*{-0.3cm}
    \centering
    \includegraphics[width=1.1\textwidth,trim= 6cm 0 4cm 0, clip]{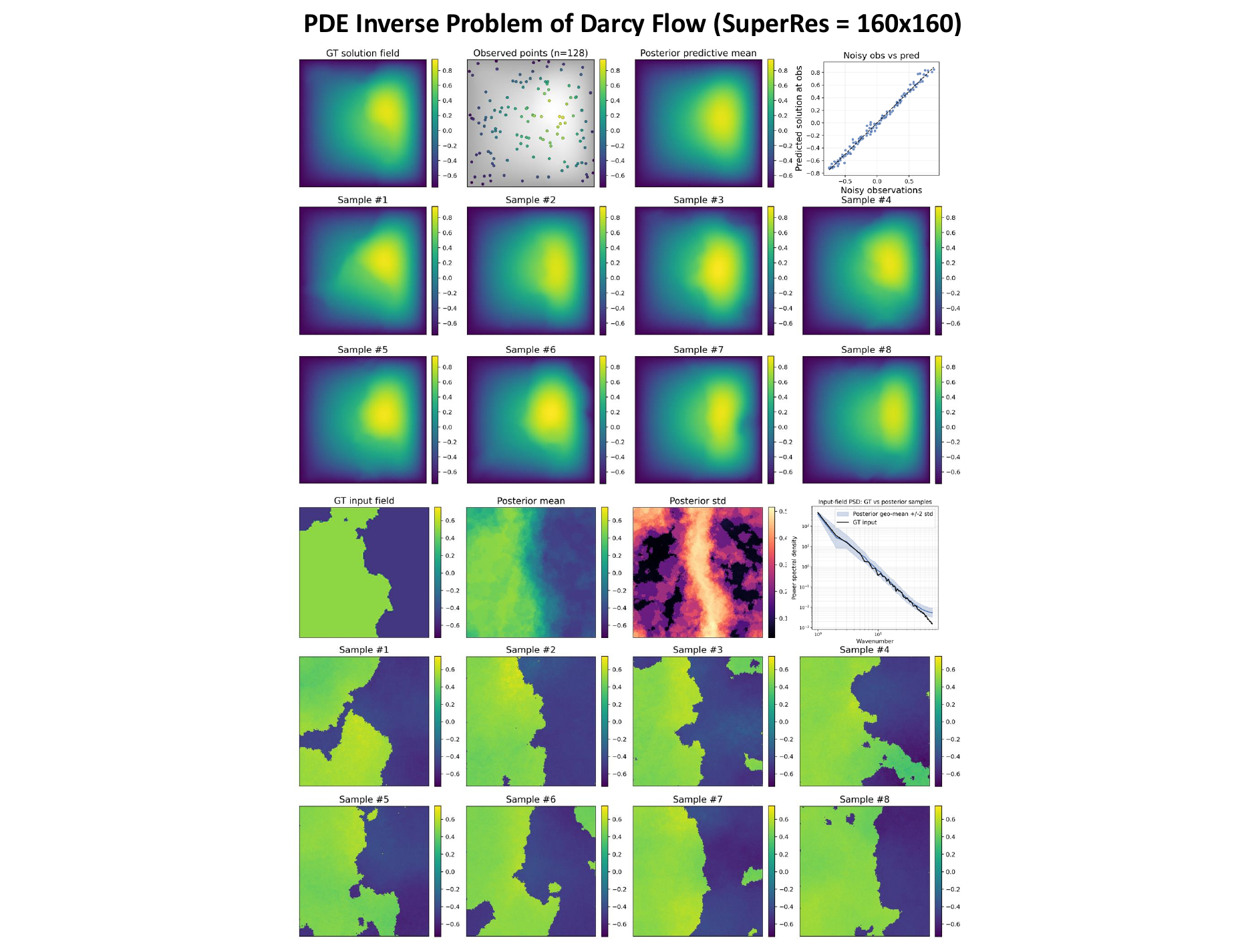}
    \caption{Zero-shot super-resolution PDE inverse problem with 128 noisy solution observation ($0.5\%$) on Darcy flow on resolution $160\times160$. FLAPS infers high-resolution input fields from sparse solution observations without retraining the prior.}
    \label{fig:sup_darcy}
\end{figure*}

\clearpage

\begin{figure*}[t]
    \vspace{-0.3cm}
    \centering
    \includegraphics[width=0.45\textwidth]{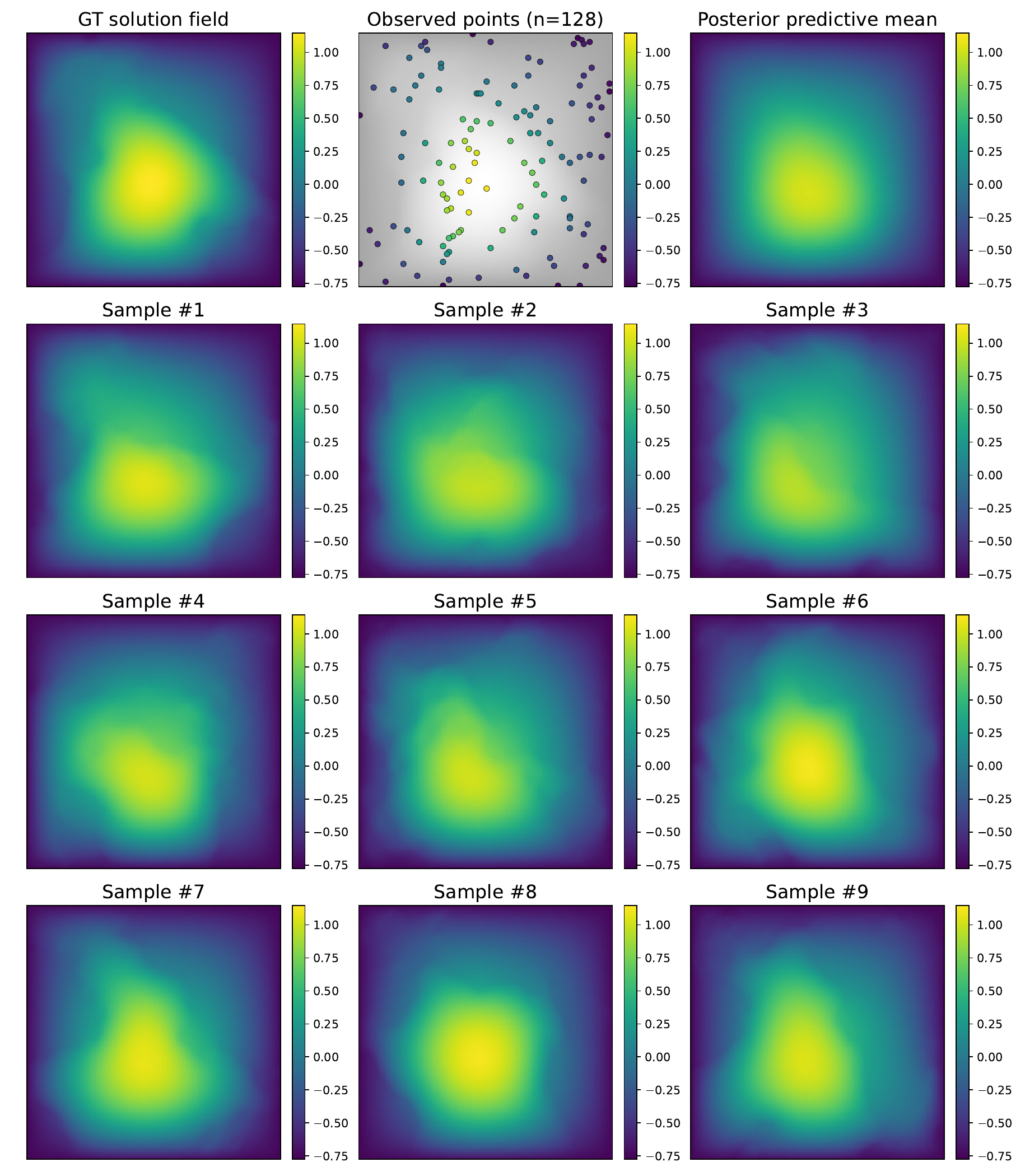}
    \quad
    \includegraphics[width=0.45\textwidth]{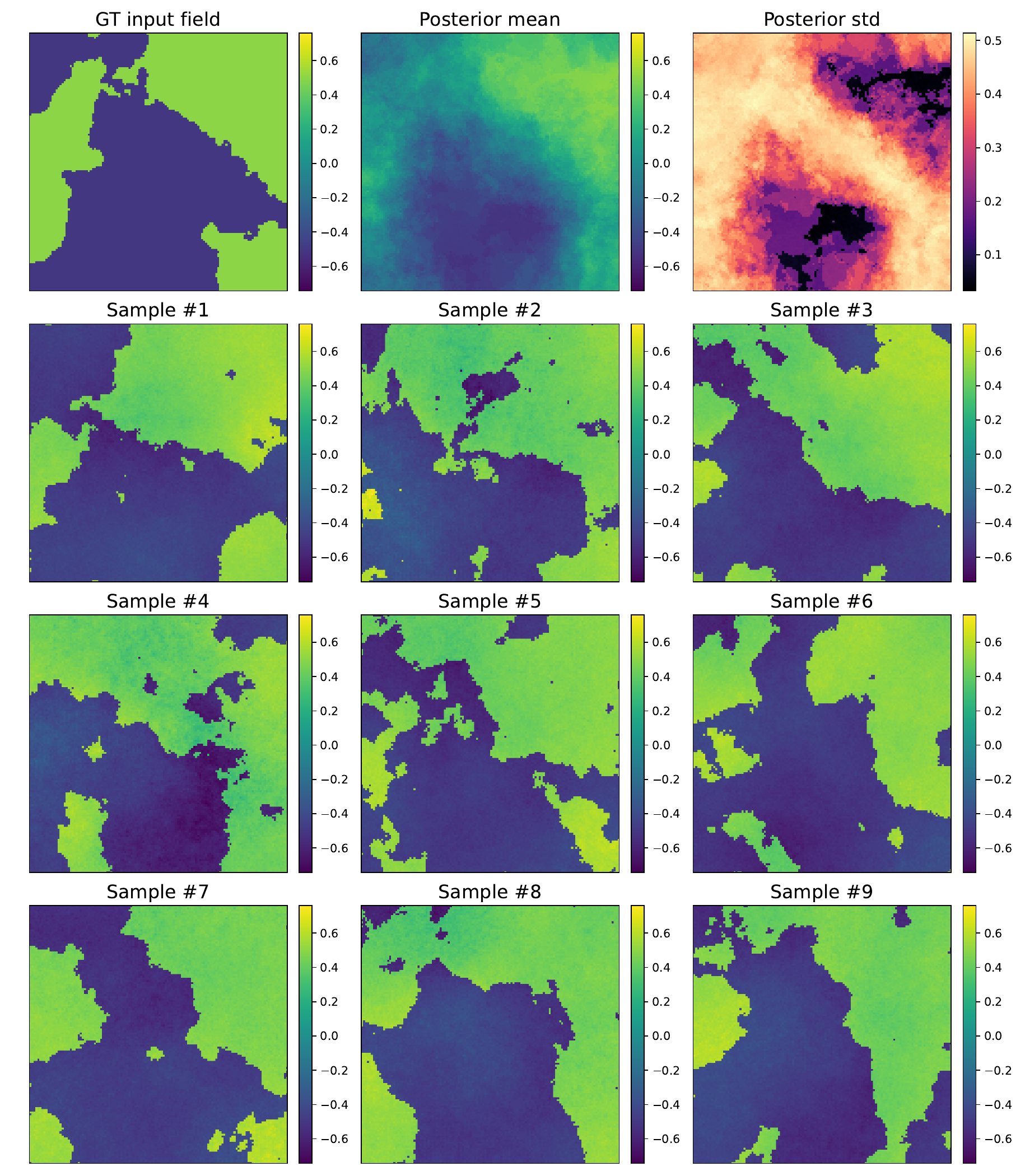}

    \caption{
    Visualization of posterior sampling for the Darcy flow inverse problem on resolution $128\times128$. 
    (\textbf{left}) the solution/output field, with the ground truth, sparse observations, posterior predictive mean, and posterior predictive samples. 
    (\textbf{right}) the input coefficient field, with the ground truth, posterior mean, posterior standard deviation, and posterior samples.
    }
    \label{fig:app_darcy}
\end{figure*}

\begin{figure*}[t]
    \vspace{-0.3cm}
    \centering
    \includegraphics[width=0.45\textwidth]{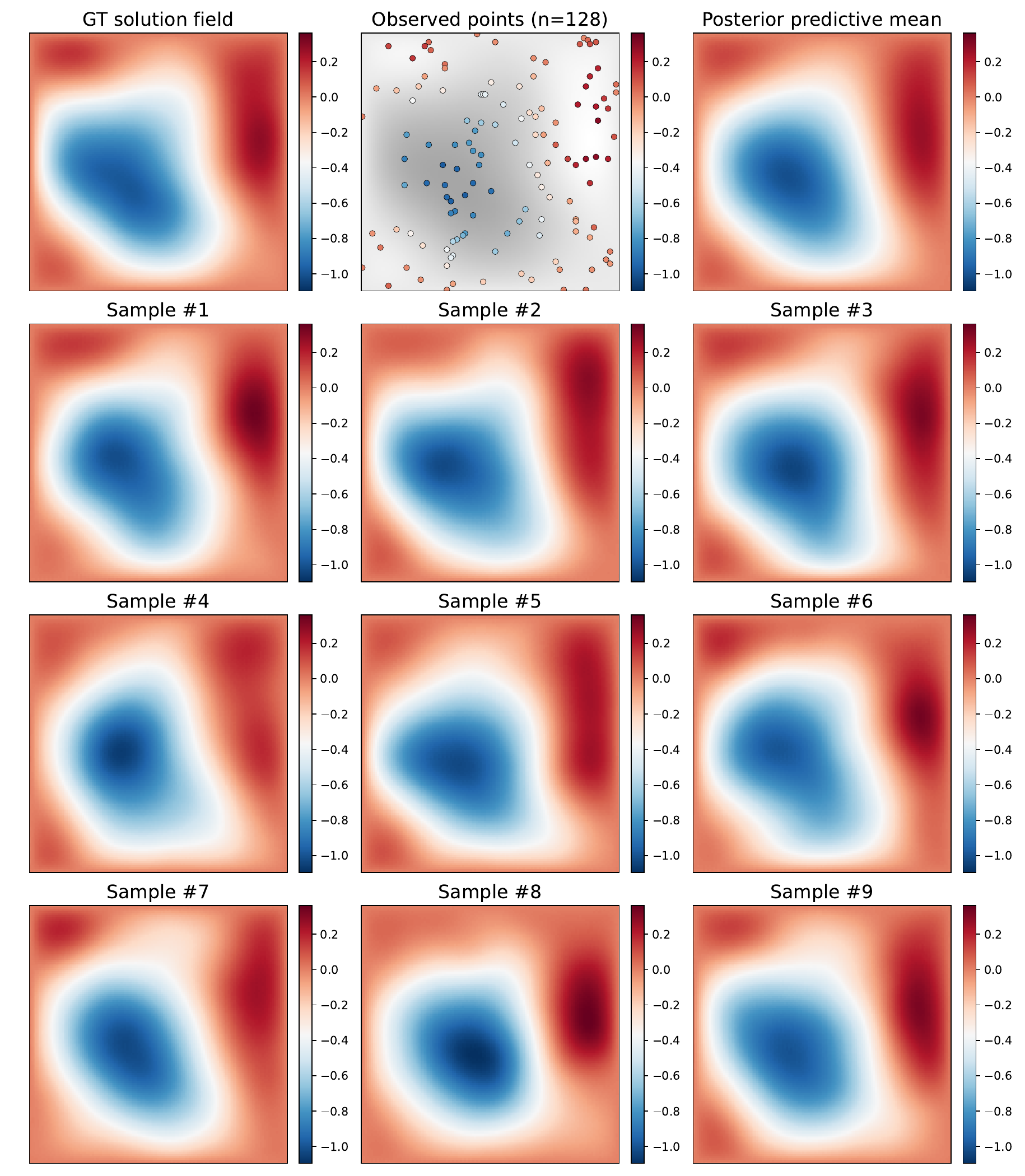}
    \quad
    \includegraphics[width=0.45\textwidth]{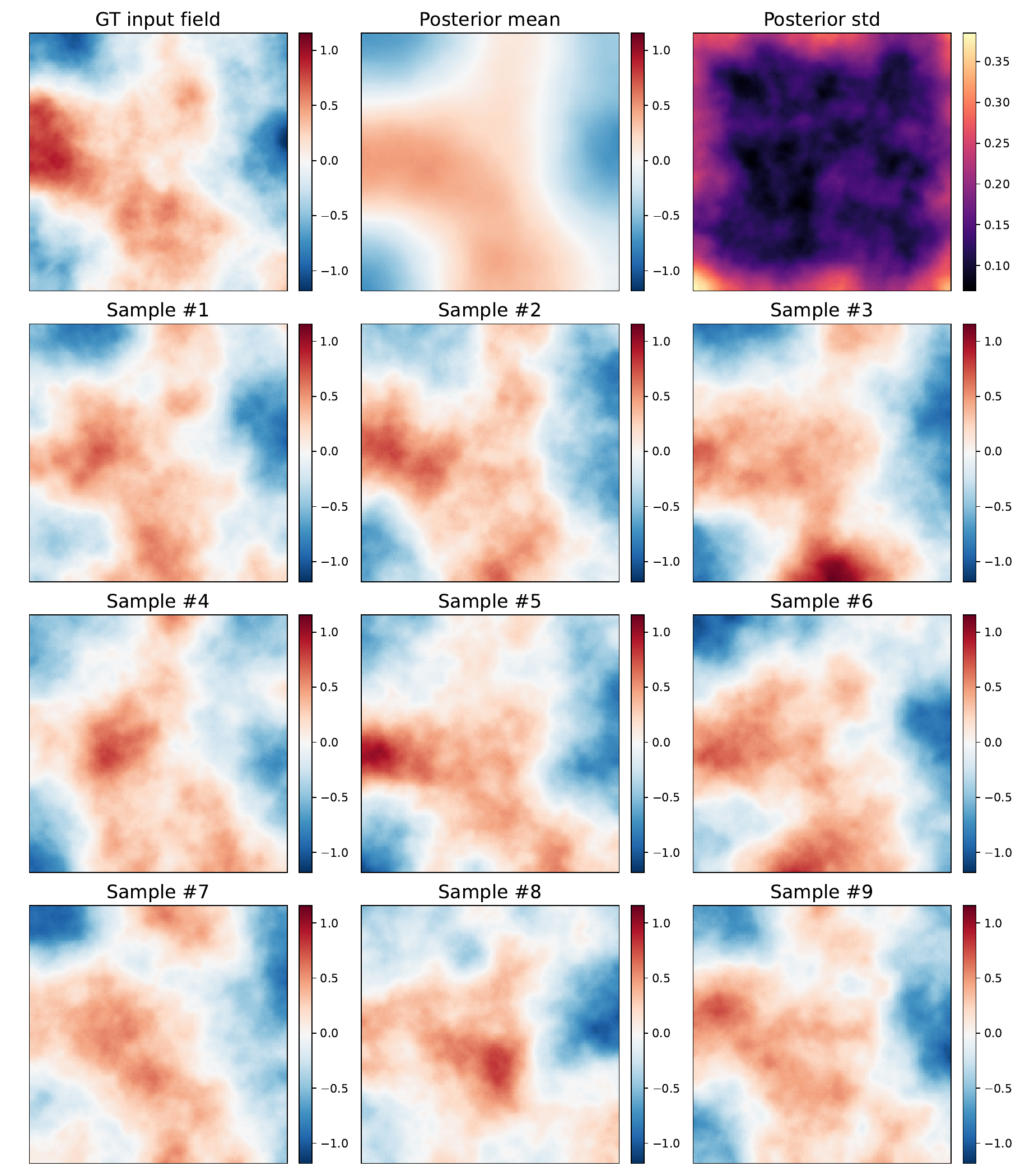}

    \caption{
    Visualization of posterior sampling for the Poisson inverse problem on resolution $128\times128$. 
    (\textbf{left}) the solution/output field, with the ground truth, sparse observations, posterior predictive mean, and posterior predictive samples. 
    (\textbf{right}) the input coefficient field, with the ground truth, posterior mean, posterior standard deviation, and posterior samples.
    }
    \label{fig:app_poisson}
\end{figure*}

\begin{figure*}[t]
    \vspace{-0.3cm}
    \centering
    \includegraphics[width=0.45\textwidth]{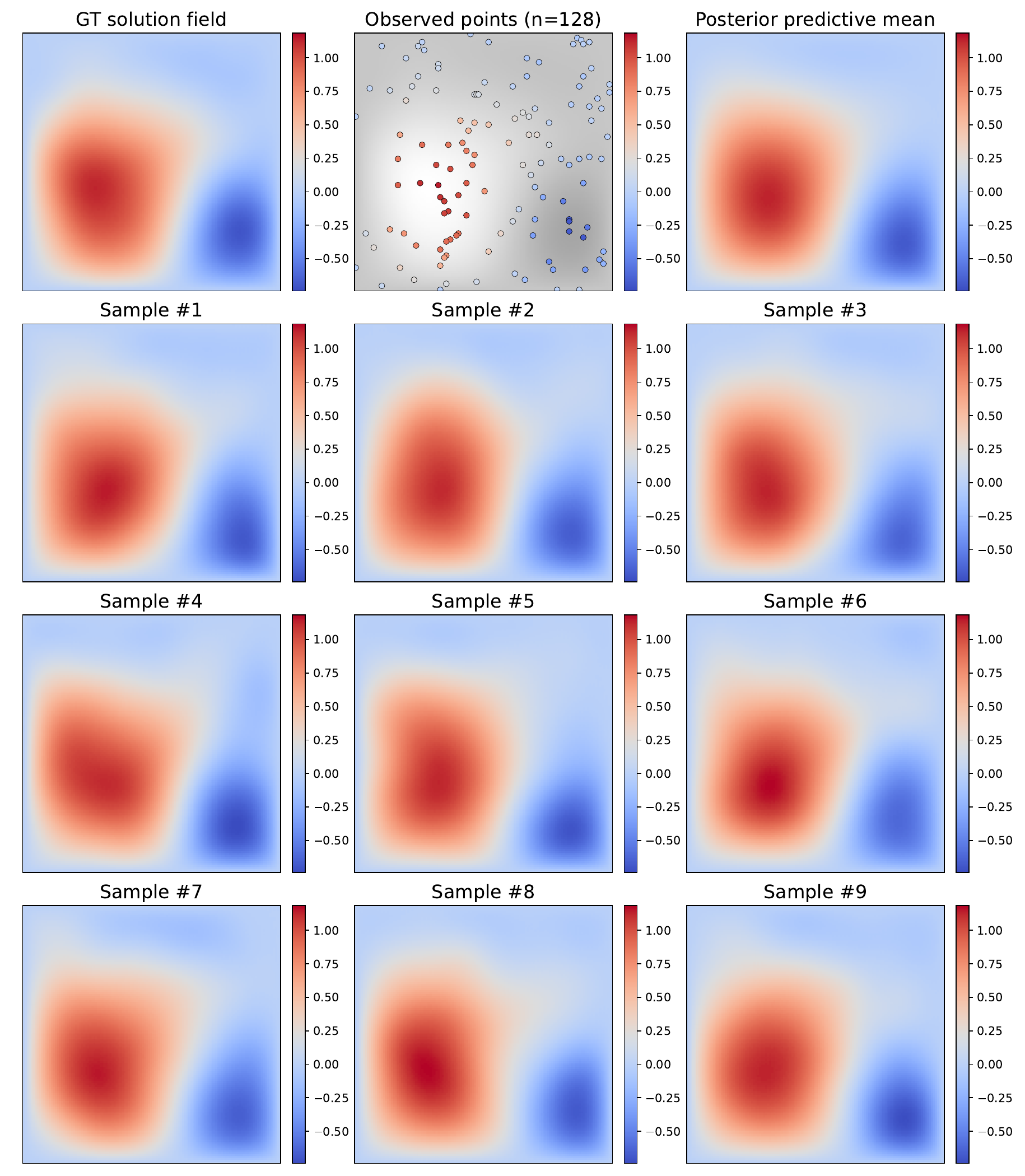}
    \quad
    \includegraphics[width=0.45\textwidth]{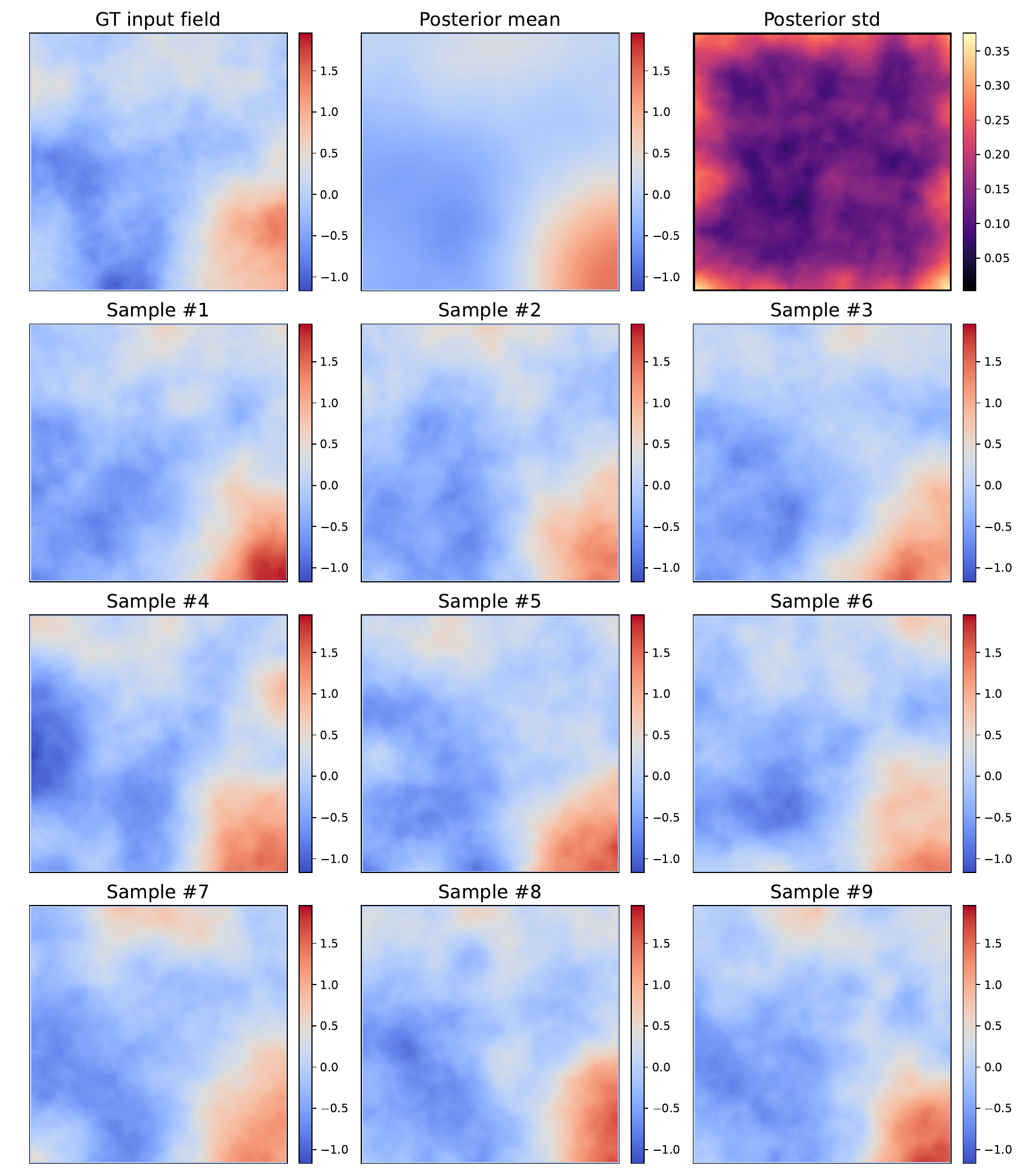}

    \caption{
    Visualization of posterior sampling for the Helmholtz inverse problem on resolution $128\times128$. 
    (\textbf{left}) the solution/output field, with the ground truth, sparse observations, posterior predictive mean, and posterior predictive samples. 
    (\textbf{right}) the input coefficient field, with the ground truth, posterior mean, posterior standard deviation, and posterior samples.
    }
    \label{fig:app_helmholtz}
\end{figure*}

\begin{figure*}[t]
    \vspace{-0.3cm}
    \centering
    \includegraphics[width=0.45\textwidth]{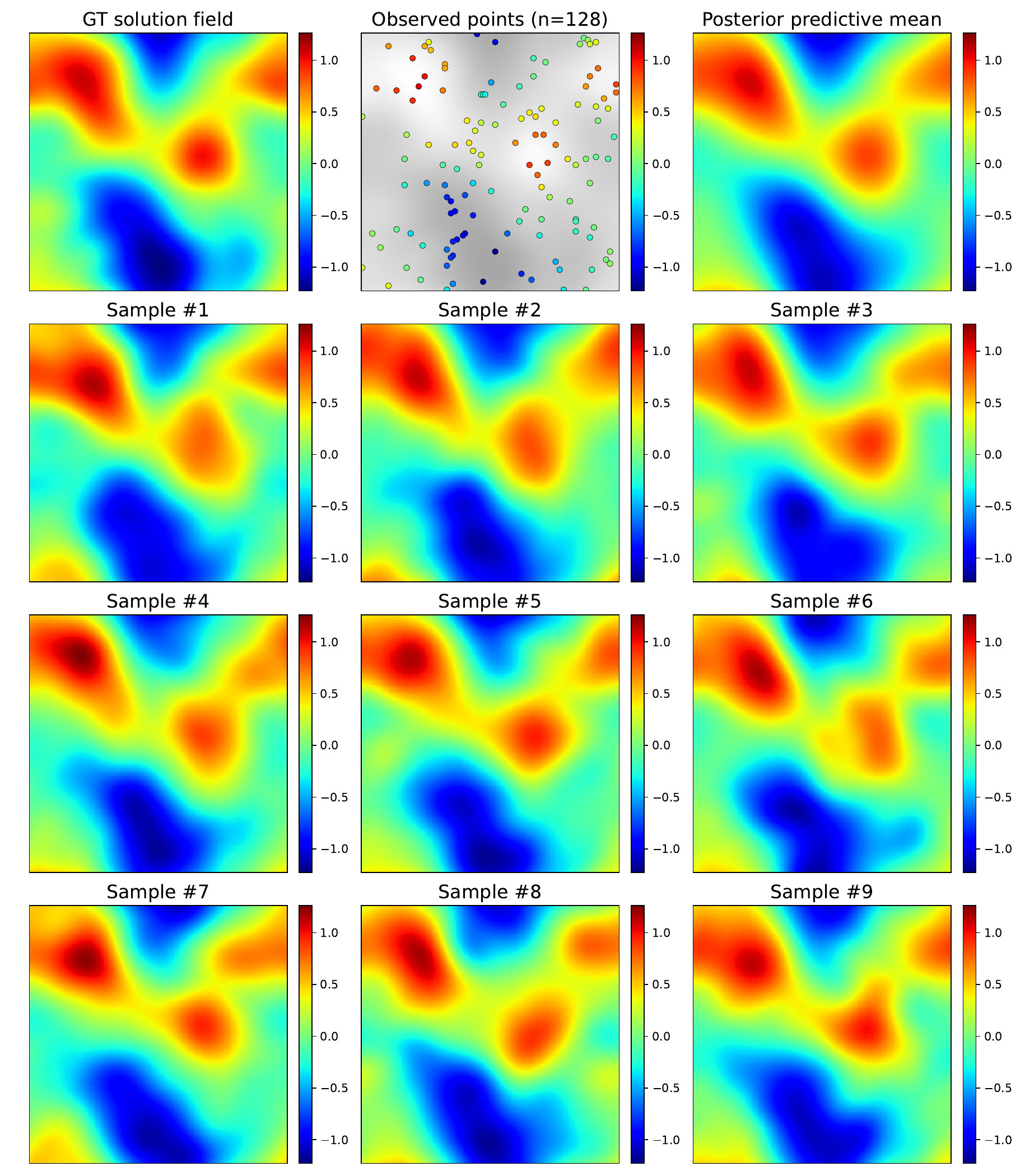}
    \quad
    \includegraphics[width=0.45\textwidth]{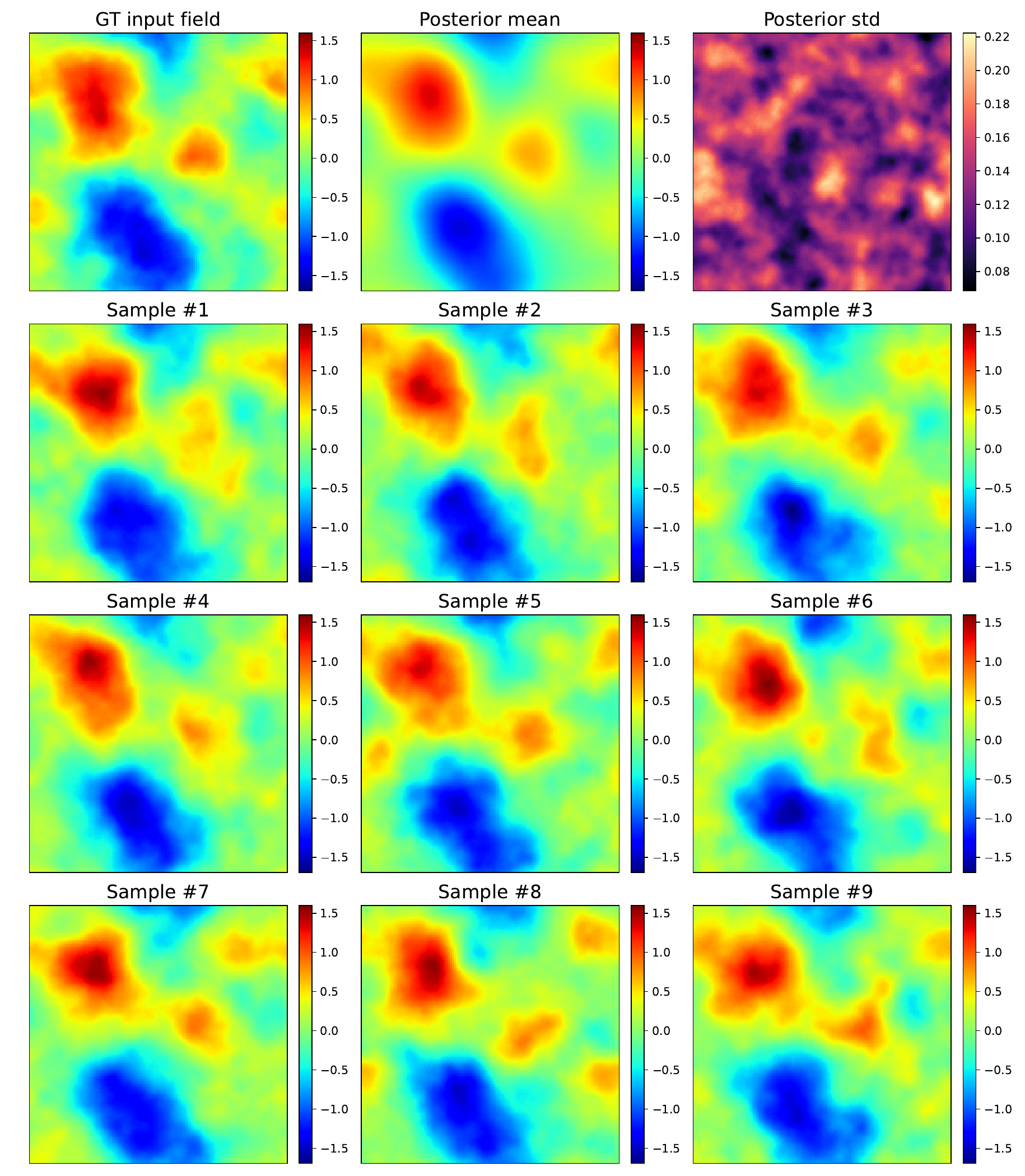}

    \caption{
    Visualization of posterior sampling for the Navier-Stokes inverse problem on resolution $128\times128$. 
    (\textbf{left}) the solution/output field, with the ground truth, sparse observations, posterior predictive mean, and posterior predictive samples. 
    (\textbf{right}) the input coefficient field, with the ground truth, posterior mean, posterior standard deviation, and posterior samples.
    }
    \label{fig:app_ns}
\end{figure*}

\clearpage
\newpage
\section*{NeurIPS Paper Checklist}


\begin{enumerate}

\item {\bf Claims}
    \item[] Question: Do the main claims made in the abstract and introduction accurately reflect the paper's contributions and scope?
    \item[] Answer: \answerYes{} 
    \item[] Justification: The abstract and introduction succinctly outline the paper’s claims and contri-
butions, which are fully supported by the theoretical analysis and experimental results.

\item {\bf Limitations}
    \item[] Question: Does the paper discuss the limitations of the work performed by the authors?
    \item[] Answer: \answerYes{} 
    \item[] Justification: We mentioned the limitations in the main text.

\item {\bf Theory assumptions and proofs}
    \item[] Question: For each theoretical result, does the paper provide the full set of assumptions and a complete (and correct) proof?
    \item[] Answer: \answerYes{} 
    \item[] Justification: See Appendix A

    \item {\bf Experimental result reproducibility}
    \item[] Question: Does the paper fully disclose all the information needed to reproduce the main experimental results of the paper to the extent that it affects the main claims and/or conclusions of the paper (regardless of whether the code and data are provided or not)?
    \item[] Answer:\answerYes{} 
    \item[] Justification: We provide detailed experimental setup in the Appendix 

\item {\bf Open access to data and code}
    \item[] Question: Does the paper provide open access to the data and code, with sufficient instructions to faithfully reproduce the main experimental results, as described in supplemental material?
    \item[] Answer: \answerYes{} 
    \item[] Justification: All experiments use publicly available datasets, and we also released the codebase and checkpoints to facilitate reproducibility.

\item {\bf Experimental setting/details}
    \item[] Question: Does the paper specify all the training and test details (e.g., data splits, hyperparameters, how they were chosen, type of optimizer) necessary to understand the results?
    \item[] Answer: \answerYes{} 
    \item[] See Appendix

\item {\bf Experiment statistical significance}
    \item[] Question: Does the paper report error bars suitably and correctly defined or other appropriate information about the statistical significance of the experiments?
    \item[] Answer: \answerYes{} 
    \item[] Justification: the metrics we use compare distribution to distribution, naturally incorporate uncertainty

\item {\bf Experiments compute resources}
    \item[] Question: For each experiment, does the paper provide sufficient information on the computer resources (type of compute workers, memory, time of execution) needed to reproduce the experiments?
    \item[] Answer: \answerYes{} 
    \item[] Justification: See Appendix
    
\item {\bf Code of ethics}
    \item[] Question: Does the research conducted in the paper conform, in every respect, with the NeurIPS Code of Ethics \url{https://neurips.cc/public/EthicsGuidelines}?
    \item[] Answer: \answerYes{} 
    \item[] Justification: The authors do not foresee any ethical implications to this work.

\item {\bf Broader impacts}
    \item[] Question: Does the paper discuss both potential positive societal impacts and negative societal impacts of the work performed?
    \item[] Answer: \answerNA{} 
    \item[] Justification: This work is not directly tied to any harmful applications.
    
\item {\bf Safeguards}
    \item[] Question: Does the paper describe safeguards that have been put in place for responsible release of data or models that have a high risk for misuse (e.g., pre-trained language models, image generators, or scraped datasets)?
    \item[] Answer: \answerNA{} 
    \item[] Justification: This is not applicable

\item {\bf Licenses for existing assets}
    \item[] Question: Are the creators or original owners of assets (e.g., code, data, models), used in the paper, properly credited and are the license and terms of use explicitly mentioned and properly respected?
    \item[] Answer: \answerYes{} 
    \item[] Justification: We cite each reused code snippet and dataset with its original publication.

\item {\bf New assets}
    \item[] Question: Are new assets introduced in the paper well documented and is the documentation provided alongside the assets?
    \item[] Answer: \answerNA{} 
    \item[] Justification: We don not release new assets

\item {\bf Crowdsourcing and research with human subjects}
    \item[] Question: For crowdsourcing experiments and research with human subjects, does the paper include the full text of instructions given to participants and screenshots, if applicable, as well as details about compensation (if any)? 
    \item[] Answer: \answerNA{} 
    \item[] Justification: We do not do crowdsourcing nor research with human subjects

\item {\bf Institutional review board (IRB) approvals or equivalent for research with human subjects}
    \item[] Question: Does the paper describe potential risks incurred by study participants, whether such risks were disclosed to the subjects, and whether Institutional Review Board (IRB) approvals (or an equivalent approval/review based on the requirements of your country or institution) were obtained?
    \item[] Answer: \answerNA{} 
    \item[] Justification: We do not do crowdsourcing nor research with human subjects

\item {\bf Declaration of LLM usage}
    \item[] Question: Does the paper describe the usage of LLMs if it is an important, original, or non-standard component of the core methods in this research? Note that if the LLM is used only for writing, editing, or formatting purposes and does \emph{not} impact the core methodology, scientific rigor, or originality of the research, declaration is not required.
    \item[] Answer: \answerNA{} 
    \item[] Justification: This work is not related to LLM

\end{enumerate}

\end{document}